\RequirePackage{fix-cm}
\documentclass[smallextended, review]{svjour3}
\usepackage{amsmath, amsxtra, amsfonts, amssymb, amstext}
\usepackage{lineno}
\usepackage{nicefrac}
\usepackage{xspace}
\usepackage{graphics,color}
\usepackage{graphicx}
\usepackage[colorlinks]{hyperref}
\definecolor{linkblue}{rgb}{0.1,0.1,0.8}
\hypersetup{colorlinks=true,linkcolor=linkblue,filecolor=linkblue,urlcolor=linkblue,citecolor=linkblue}
\usepackage{float}
\usepackage{etoolbox}
\usepackage{cite}

\usepackage[ruled,vlined,linesnumbered]{algorithm2e}

\usepackage{footnote}
\usepackage{pdfpages,caption}
\usepackage{enumerate}
\usepackage{wrapfig}

\smartqed  % flush right qed marks, e.g. at end of proof

\usepackage{tikz}
\usetikzlibrary{trees,positioning,shapes,arrows}
\tikzstyle{line}=[draw]
\tikzstyle{level 1}=[level distance=3.5cm, sibling distance=6cm]
\tikzstyle{level 2}=[level distance=5cm, sibling distance=4cm]
\tikzstyle{level 3}=[level distance=7cm, sibling distance=6cm]
\tikzstyle{level 4}=[level distance=7cm, sibling distance=4cm]

\tikzstyle{bag} = [text width=4em, text centered]
\tikzstyle{end} = [circle, minimum width=3pt,fill, inner sep=0pt]

\newcommand{\NN}{\mathbb{N}}
\newcommand{\RR}{\mathbb{R}}

\newcommand{\E}{\mathbb E}
\renewcommand{\epsilon}{\varepsilon}
\newcommand{\eps}{\varepsilon}

\newcommand{\Geom}{\textsc{Geom}\xspace}

\newcommand{\ooea}{\ensuremath{(1 + 1)}\text{-EA}\xspace}
\newcommand{\toea}{\ensuremath{(2 + 1)}\text{-EA}\xspace}

\newcommand{\olea}{\ensuremath{(1 + \lambda)}\text{-EA}\xspace}
\newcommand{\moea}{\ensuremath{(\mu + 1)}\text{-EA}\xspace}
\newcommand{\moga}{\ensuremath{(\mu + 1)}\text{-GA}\xspace}

\newcommand{\ollga}{\ensuremath{(1 + (\lambda,\lambda))}\text{-GA}\xspace}

\newcommand{\folea}{\ensuremath{(1 + \lambda)}\text{-fEA}\xspace}

\newcommand{\DeltaEA}{\Delta}

\newcommand{\Deltacon}{\Delta_\text{con}}

\newcommand{\OPT}{\ensuremath{\textsc{OPT}}}

\newcommand{\BinVal}{\textsc{BinVal}\xspace}

\newcommand{\dynbv}{\textsc{DynBV}\xspace}

\newcommand{\hottopic}{\textsc{HotTopic}\xspace}

\newtoggle{journal}
\toggletrue{journal}
{}

\DeclareMathOperator{\EX}{\mathbb{E}} % expected value
\DeclareMathOperator*{\argmax}{arg\,max}
\DeclareMathOperator*{\argmin}{arg\,min}
\usepackage{hyperref}
\usepackage{tikz}\usetikzlibrary{automata, positioning, arrows}
\tikzset{->,  
edge/.style = {->,> = latex'}
node distance=3cm, 
every node/.style={sloped,anchor=south,auto=false},
initial text=$ $,
}
\hypersetup{
    colorlinks,
    citecolor=black,
    filecolor=black,
    linkcolor=black,
    urlcolor=black
}

\begin{document}

\title{Runtime analysis of the \moea on the Dynamic BinVal function}
\titlerunning{\moea on the Dynamic BinVal Function}
%\titlenote{Produces the permission block, and copyright information}
%\subtitle{Extended Abstract}

%\author{Johannes Lengler}
%%\authornote{Dr.~Trovato insisted his name be first.}
%%\orcid{1234-5678-9012}
%\affiliation{%
%  \institution{ETH Z{\"u}rich, Switzerland}
%%  \streetaddress{P.O. Box 1212}
%%  \city{Dublin} 
%% \state{Ohio} 
%% \postcode{43017-6221}
%}
\author{Johannes Lengler \and Simone Riedi}
\authorrunning{J. Lengler, S. Riedi}
\institute{Department of Computer Science \\ETH Z{\"u}rich, Z{\"u}rich, Switzerland}
\date{Received: date / Accepted: date}
\maketitle

% \begin{CCSXML}
%<ccs2012>
%<concept>
%<concept_id>10003752.10010070.10011796</concept_id>
%<concept_desc>Theory of computation~Theory of randomized search heuristics</concept_desc>
%<concept_significance>500</concept_significance>
%</concept>
%</ccs2012>
%\end{CCSXML}
%
%\ccsdesc[500]{Theory of computation~Theory of randomized search heuristics}
%
%%\category{F.2.2}{Theory of Computation}{Analysis of Algorithms and Problem Complexity}[Nonnumerical Algorithms and Problems]
%\keywords{Theory, Runtime Analysis, Monotone Functions, Crossover, Mutation Strength, HotTopic}
%%\textbf{Category:} {F.2.2}{Theory of Computation}{Analysis of Algorithms and Problem Complexity}[Nonnumerical Algorithms and Problems]\\
%%
%%\textbf{Keywords:} {Black-Box Complexity, Elitist Selection, Comparison-Based Algorithms}
%

%\newcommand{\package}{\emph}

%% The abstract of your thesis.  Edit the file as needed.
\begin{abstract}
We study evolutionary algorithms in a dynamic setting, where for each generation a different fitness function is chosen, and selection is performed with respect to the current fitness function. Specifically, we consider Dynamic BinVal, in which the fitness functions for each generation is given by the linear function BinVal, but in each generation the order of bits is randomly permuted. For the \ooea it was known that there is an efficiency threshold $c_0$ for the mutation parameter, at which the runtime switches from quasilinear to exponential. Previous empirical evidence suggested that for larger population size $\mu$, the threshold may increase. We prove that this is at least the case in an $\eps$-neighborhood around the optimum: the threshold of the \moea becomes arbitrarily large if the $\mu$ is chosen large enough.

However, the most surprising result is obtained by a second order analysis for $\mu=2$: the threshold \emph{in}creases with increasing proximity to the optimum. In particular, the hardest region for optimization is \emph{not} around the optimum.\footnote{An extended abstract of the paper has appeared in the proceedings of the EvoCOP conference~\cite{lengler2020runtime}. The extended abstract was missing many proofs, in particular the whole derivation of the second-order expansion in Section~\ref{sec:second_order}.}
%On behalf of all authors, the corresponding author states that there is no conflict of interest. 
\end{abstract}

\section{Introduction}
Evolutionary algorithms are optimization heuristics that are based on the idea of maintaining a population of solutions that evolves over time. This incremental nature is an important advantage of population-based optimization heuristics over non-incremental approaches. At any point in time the population represents a set of solutions. This makes population-based optimization heuristics very flexible. For example, the heuristic can be stopped after any time budget (predefined or chosen during execution), or when some desired quality of the solutions is reached. For the same reason, population-based algorithms are naturally suited for dynamic environments, in which the optimization goal (``fitness function'') may change over time. In such a setting, it is not necessary to restart the algorithm from scratch when the fitness function changes, but rather we can use the current population as starting point for the new optimization environment. If the fitness function changes slowly enough, then population-based optimization heuristics may still find the optimum, or track the optimum over time~\cite{dang2018new,droste2002analysis,kotzing2015generalized,kotzing2012aco,lissovoi2016mmas,lissovoi2017runtime,lissovoi2018impact,neumann2015runtime,pourhassan2015maintaining,shi2019reoptimization}. We refrain from giving a detailed overview over the literature since an excellent review has recently been given by Neumann, Pourhassan, and Roostapour~\cite{neumann2020analysis}. All the settings have in common that either the fitness function changes with very low frequency, or it changes only by some small local differences, or both. 

Recently, a new setting, called \emph{dynamic linear functions} was proposed by Lengler and Schaller~\cite{lengler2018noisy}. They argued that it might either be called noisy linear functions or dynamic linear functions, but we prefer the term dynamic. A class of dynamic linear functions is determined by a distribution $\mathcal D$ on the positive reals $\RR^+$. For the $k$-th generation, $n$ weights $W_1^{k},\ldots,W_n^{k}$ are chosen independently identically distributed (i.i.d.) from $\mathcal D$, and the fitness function for this generation is given by $f^{k}: \{0,1\}^n \to \RR^+;\; f^{k}(x) = \sum_{i=1}^n W_i^{k}x_i$. So the fitness in each generation is given by a linear function with positive weights, but the weights are drawn randomly in each generation. Note that for any fitness function, a one-bit in the $i$-th position will always yield a better fitness than a zero-bit. In particular, all fitness functions share a common global maximum, which is the string $\OPT = (1...1)$. Hence, the fitness function may change rapidly and strongly from generation to generation, but the direction of the signal remains unchanged: one-bits are preferred over zero-bits. 

Crucially, by dynamic environments we mean that \emph{selection} is performed according to the current fitness function as in~\cite{horoba2010ant}. I.e., all individuals from parent and offspring population are compared with respect to \emph{the same} fitness function. Other versions exist, e.g.~\cite{doerr2012ants} studies the same problem as~\cite{horoba2010ant} without re-evaluations, i.e., there algorithms would compare fitnesses like $f^k (x)$ and $f^{k+1}(y)$ with each other, which will never happen in our setting.

Several applications of dynamic linear functions are discussed in~\cite{lengler2018noisy}. One of them is a chess engine that can switch databases for different openings ON or OFF. The databases strictly improve performance in all situations, but if the engine is trained against varying opponents, then an opening may be used more or less frequently; so the weight of the corresponding bit my be high or low. Obviously, it is desirable that an optimization heuristic manages to switch all databases ON in such a situation. However, as we will see, this is not automatically achieved by many simple optimization heuristics. Rather, it depends on the parameter settings whether the optimal configuration (all databases ON) is found.

In~\cite{lengler2018noisy}, the runtime (measured as the number of iterations until the optimum is found) of the well-known \ooea on dynamic linear functions was studied. The \ooea, or ``$(1+1)$ Evolutionary Algorithm'', is a simple hillclimbing algorithm for maximizing a pseudo-Boolean function $f:\{0,1\}^n \to \RR$. It only maintains a population size of $\mu=1$, so it maintains a single solution $x^{k} \in \{0,1\}^n$. In each round (also called \emph{generation}), a randomized \emph{mutations operator} is applied to $x^k$ to generate an \emph{offspring} $y^{k}$. Then the \emph{fitter} of the two is maintained, so we set $x^{k+1} := x^{k}$ if $f(x^k) > f(y^k)$, and $x^{k+1} := y^{k}$ if $f(x^k) < f(y^k)$. In case of equality, we break ties randomly. The mutation operator of the \ooea is \emph{standard bit mutation}, which flips each bit of $x^k$ independently with probability $c/n$, where $c$ is called the \emph{mutation parameter}. The authors of~\cite{lengler2018noisy} gave a full characterization of the optimization behavior of the \ooea on dynamic linear functions in terms of the mutation parameter~$c$. It was shown that there is a threshold $c^* = c^*(\mathcal D) \in \RR^+ \cup \{\infty\}$ such that for $c < c^*$ the \ooea optimizes the dynamic linear function with weight distribution $\mathcal D$ in time $O(n \log n)$. On the other hand, for $c > c^*$, the algorithm needs exponential time to find the optimum. The threshold $c^* (\mathcal D)$ was given by an explicit formula. For example, if $\mathcal D$ is an exponential distribution then $c^*(\mathcal D) = 2$, if it is a geometric distribution $\mathcal D = \Geom(p)$ then $c^* = (2-p)/(1-p)$. Moreover, the authors in~\cite{lengler2018noisy} showed that there is $c_0 \approx 1.59..$ such that $c^*(\mathcal D) > c_0$ for every distribution $\mathcal D$, but for any $\eps > 0$ there is a distribution $\mathcal D$ with $c^*(\mathcal D) < c_0+\eps$. As a consequence, if $c < c_0$ then the $\ooea$ with mutation parameter $c/n$ needs time $O(n \log n)$ to optimize any dynamic linear function, while for $c>c_0$ there are dynamic linear functions on which it needs exponential time.

While it was satisfying to have such a complete picture for the \ooea, a severe limitation was that the \ooea is very simplistic. In particular, it was unclear whether a non-trivial population size $\mu >1$ would give a similar picture. This question was considered in the experimental paper~\cite{lengler2020large,lengler2020fullversion} by Lengler and Meier. Instead of working with the whole class of dynamic linear functions, they defined the \emph{dynamic binary value function} \dynbv as a limiting case. In \dynbv, in each generation a uniformly random permutation $\pi^{k}:\{1,\ldots,n\} \to \{1,\ldots,n\}$ of the bits is drawn, and the fitness function is then given by $f^{k}(x) = \sum_{i=1}^n 2^{n-i}x_{\pi^{k}(i)}$. So in each generation, \dynbv evaluates the \BinVal function with respect to a permutation of the search space. Lengler and Meier observed that the proof in~\cite{lengler2018noisy} for the \ooea extends to \dynbv with threshold $c^* = c_0$, i.e., the \ooea needs time $O(n \log n)$ for mutation parameter $c<c_0$, and exponential time for $c > c_0$. In this sense, \dynbv is the hardest dynamic linear function, although it is not formally a member of the class of dynamic linear functions.

The papers~\cite{lengler2020large,lengler2020fullversion} performed experiments on\dynbv for two population-based algorithms, the \moea (using only mutation) and \moga (using randomly mutation or crossover; GA stands for ``Genetic Algorithm''). In $(\mu+1)$ algorithms, a population of size $\mu$ is maintained, see also Algorithm~\ref{alg:generic}. In each generation, a single offspring is generated, and the least fit of the $\mu+1$ search points is discarded, breaking ties randomly. Thus they generalize the \ooea. In the \moea, the offspring is generated by picking a random \emph{parent} from the population and performing standard bit mutation as in the \ooea. In the \moga, it is also possible to generate the offspring by \emph{crossover}: two random parents $x_1, x_2$ are selected from the population, and each bit is taken randomly either from $x_1$ or $x_2$. In each generation of the \moga, it is decided randomly with probability $1/2$ whether the offspring is produced by a mutation or by a crossover.\footnote{Other conventions are possible, e.g. that both crossover and mutation are applied subsequently in the same generation. Here we describe the version in \cite{lengler2020large} and \cite{lengler2020fullversion}.}

\begin{algorithm}
Initialize $P^0$ with $\mu$ strings chosen u.a.r. from $\{0,1\}^n$\;
\For{$k=0,1,2,3...$}{
    From $P^{k}$, generate $\lambda$ offspring $\hat{x}_1,..,\hat{x}_{\lambda}$. \; % by \emph{mutation} or by \emph{crossover}.\; 
    $S \leftarrow P^{k} \cup \{\hat{x}_1,..,\hat{x}_{\lambda}\}$\;
    Based on the current fitnesses $f^{k}(x)$ for $x \in S$, \emph{select} $\mu$ individuals from $S$ to obtain $P^{k+1}$
	}
%\BlankLine
%\emph{mutation:} flip each bit of a random parent independently with prob. $c/n$.\;
%\emph{crossover:} choose independently two random parents; for each position flip a fair coin to select between the two parent bits in this position.\;
%\emph{selection:} choose greedily the fittest individuals, breaking ties randomly.\;
\caption{A generic $(\mu+\lambda)$ algorithm in dynamic environments. In this paper, we use the \moea. That means that we use $\lambda =1$, standard bit mutation with mutation parameter $c$ (mutation rate $c/n$), and elitist selection (greedy selection by fitness).\label{alg:generic}
%with mutation parameter $c$. The last lines specify mutation/\-crossover/selection operator used in this paper. Also, we only consider $\lambda =1$. The EA uses mutation only. For the GA, each offspring is either generated by mutation or by crossover, each with probability $1/2$. 
}
\end{algorithm}

Lengler and Meier ran experiments for $\mu \in \{1,2,3,5\}$ on \dynbv and found two main results. As they increased the population size $\mu$ from $1$ to $5$, the efficiency threshold $c_0$ increased moderately for the \moea (from $1.6$ to $3.4$, and strongly for the \moga (from $1.6$ to more than $20$). So with larger population size, the algorithms have a larger range of feasible parameter settings, and even more so when crossover is used. 

Moreover, they studied which range of the search space was hardest for the algorithms, by estimating the drift towards the optimum with Monte Carlo simulations. For the \moga, they found that the hardest region was around the optimum, as one would expect. Surprisingly, for the \moea with $\mu \ge 2$, this did not seem to be the case. They gave empirical evidence that the hardest regime was bounded away from the optimum. I.e., there were parameters $c$ for which the \moea had positive drift (towards the optimum) in a region around the optimum. But it had \emph{negative} drift in an intermediate region that was further away from the optimum. This finding is remarkable since it contradicts the commonplace that optimization gets harder closer to the optimum. Notably, a very similar phenomenon was proven by Lengler and Zou~\cite{lengler2019exponential} for the \moea on certain monotone functions (``\hottopic''), see the discussion below. Strikingly, such an effect was neither built into the fitness environments (not for \hottopic, and not for \dynbv) nor into the algorithms. Rather, it seems to originate in a complex (and detrimental!) population dynamics that unfolds only in a regime of weak selective pressure. If selective pressure is strong, then the population often degenerates into copies of the same search point. As a consequence, diversity is lost, and the \moea degenerates into the \ooea. In these regimes, diversity \emph{decreases} the ability of the algorithms to make progress. For \hottopic functions, these dynamics are well-understood~\cite{lengler2019general,lengler2019exponential}. For dynamic linear functions, even though we can prove this behavior in this paper for the \toea (see below), we are still far from a real understanding of these dynamics. Most likely, they are different from the dynamics for \hottopic functions.

\subsubsection{Our Results.}

We complement the experiments in~\cite{lengler2020large,lengler2020fullversion} with rigorous mathematical analysis. To this end, we study the \emph{degenerate population drift} (see Section~\ref{sec:prelim}) for the \moea with mutation parameter $c>0$ on \dynbv in an $\eps$-neighbourhood of the optimum. I.e., we assume that the search points in the current population have at least $(1-\eps)n$ one-bits, for some sufficiently small constant $\eps > 0$. We find that for every constant $c>0$ there is a constant $\mu_0$ such that for $\mu \geq \mu_0$ the drift is positive (multiplicative drift towards the optimum). This means that with high probability the \moea will need time $O(n\log n)$ to improve from $(1-\eps)n$ one-bits to the optimum, if $\mu$ is large enough. This implies that larger population sizes are helpful, since the drift of the \ooea around the optimum is negative for all $c > c_0 \approx 1.59..$ (which implies exponential optimization time). So for any $c > c_0$, increasing the population size to  a large constant decreases the runtime from exponential to quasi-linear. This is consistent with the experimental findings in~\cite{lengler2020large} for $\mu =\{1,2,3,5\}$, and it proves that population size can compensate for arbitrarily large mutation parameters.

For the \toea, we perform a second-order analysis (i.e., we determine not just the main order term of the drift, but also the second-order term) and prove that in an $\eps$-neighborhood of the optimum, the drift decreases with the distance from the optimum. In particular, there are some values of $c$ for which the drift is positive around the optimum, but negative in an intermediate distance. It follows from standard arguments that there are $\eps,c >0$ such that the runtime is $O(n\log n)$ if the algorithm is started in an $\eps$-neighborhood of the optimum, but that it takes exponential time to reach this $\eps$-neighborhood. Thus we formally prove that the hardest part of optimization is not around the optimum, as was already experimentally concluded from Monte Carlo simulations in~\cite{lengler2020large}.

\subsubsection{Related Work.}

Jansen~\cite{jansen2007brittleness} introduced a pessimistic model for analyzing the \ooea on linear functions, later extended in~\cite{colin2014monotonic}, which is \emph{also} a pessimistic model for dynamic linear functions and \dynbv \emph{and} for monotone functions. A monotone function $f:\{0,1\}^n \to \RR$ is a function where for every $x \in \{0,1\}^n$, the fitness of $x$ strictly increases if we flip any zero-bit of $x$ into a one-bit. Thus, as for dynamic linear functions and \dynbv, a one-bit is always better than a zero-bit, the optimum is always at $(1,\ldots,1)$, and there are short fitness-increasing paths from any search point to the optimum. Thus it is reasonable to call all these setting ``easy'' from an optimization point of view, which makes it all the more surprising that such a large number of standard optimization heuristics fail so badly. Keep in mind that despite the superficial similarities between monotone functions and \dynbv or dynamic linear functions, the basic setting is rather different. Monotone functions were studied in static settings, i.e., we have only a single static function to optimize, and a search point never changes its fitness. Nevertheless, the performance of some algorithms is surprisingly similar on monotone functions and on dynamic linear functions or \dynbv. In particular, the mutation parameter~$c$ plays a critical role in both settings. It was shown in~\cite{doerr2013mutation} that the \ooea needs exponential time to optimize some monotone functions if the mutation parameter $c$ is too large, while it is efficient on all monotone functions if $c<1$.\footnote{This was later extended to $c <1+\eps$ in~\cite{lengler2019does}.} The construction of hard monotone instances was simplied in~\cite{lengler2018drift} and later called \hottopic functions. \hottopic functions were analyzed for a large set of algorithms in~\cite{lengler2019general}. For the \olea, the \ollga, the \moea, and the \folea, thresholds for the mutation parameter $c$ or related quantities were determined such that a larger mutation rate leads to exponential runtime, and a smaller mutation rate leads to runtime $O(n\log n)$. (For details on these algorithms, see~\cite{lengler2019general}.) Interestingly, the population size $\mu$ and offspring population size $\lambda$ of the algorithms had no impact on the threshold. Crucially, all these results were obtained for parameters of \hottopic functions in which only the behavior in an $\eps$-neighborhood around the optimum mattered. This dichotomy between quasilinear and exponential runtime is very similar to the situation for \dynbv. However, for the \moea on \hottopic functions the threshold $c_0$ was independent of $\mu$, while we show that on \dynbv it becomes arbitrarily large as $\mu$ grows. Thus large population sizes help for \dynbv, but not for \hottopic. 
%For the \moga, population size helps both on \hottopic and on \dynbv: increasing the population size can shift the threshold $c_0$ to arbitrarily large values. Again, all these statements only apply to the behavior around the optimum.

As we prove, for the \toea the region around the optimum is not the hardest region for optimization, and there are values of $c$ for which there is a positive drift around the optimum, but a negative drift in an intermediate region. As Lengler and Zou showed~\cite{lengler2019exponential}, the same phenomenon occurs for the \moea on \hottopic functions. In fact, they showed that larger population size even hurts: for any $c>0$ there is a $\mu_0$ such that the $\moea$ with $\mu \geq \mu_0$ has negative drift in some intermediate region (and thus exponential runtime), even if $c$ is much smaller than one! This surprising effect is due to population dynamics in which it is not the genes of the fittest individuals who survive in the long terms. Rather, individuals which are strictly dominated by others (and substantially less fit) serve as the seeds for new generations. Importantly, the analysis of this dynamics relies on the fact that for \hottopic functions, the weight of the positions stay fixed for a rather long period of time (as long as the algorithm stays in the same region/level of the search space). Thus, the results do not transfer to \dynbv functions. Nevertheless, the picture looks similar insofar as the hardest region for optimization is not around the optimum in both cases. Since our analysis for \dynbv is only for $\mu=2$, we can't say whether the efficiency threshold in $c$ is increasing or decreasing with $\mu$. The experiments in~\cite{lengler2020large,lengler2020fullversion} find increasing thresholds (so the opposite effect as for \hottopic), but are only for $\mu \leq 5$.

%For the \moga, the question of whether the region around the optimum is hardest, and whether increased population size is overall beneficial or detrimental, remains widely open, both for \dynbv and for \hottopic. The experiments in~\cite{lengler2020large,lengler2020fullversion} find it to be extremely beneficial in the range of $\mu \in\{1,2,3,5\}$ for \dynbv, but it remains open for larger values of $\mu$. Thus it is also unclear whether the \moga is generally better than the \moea for \dynbv and \hottopic, although we have some scattered indications for it: experimental results for small $\mu$, the behavior around the optimum for \hottopic, and experimental hints that the most difficult region for \dynbv is around the optimum. In general, the authors believe that understanding the population dynamics in the three open cases (\moea for \dynbv, \moga for \dynbv and \hottopic), and more generally understanding non-trivial population dynamics, is one of the most important questions in the area of evolutionary computation.

%For $\mu = 2$, we show that the threshold for the GA (at the optimum) is really larger than for the EA. So the GA is really better than the EA for $\mu=2$. This confirms previous experimental findings. 

%%%%%%%%%%%%%%%%%%%%% PRELIMINARIES %%%%%%%%%%%%%%%%%%

\section{Preliminaries}\label{sec:prelim}
%In this bachelor thesis, we resume where the authors in \cite{lengler2020large} left off: we continue the study of the BinVal function, mainly in a situation close to the optimum. 
\subsection{Dynamic optimization and the dynamic Binary Value function \dynbv}\label{sec:dynamicopt}

The general setting of a $(\mu+\lambda)$ algorithm in dynamic environments on the hypercube $\{0,1\}^n$ is as follows. A population $P^k$ of $\mu$ search points is maintained. In each generation $k$, $\lambda$ offspring are generated. Then a \emph{selection operator} selects the next population $P^{k+1}$ from the $\mu+\lambda$ search points according to the fitness function $f^k$. 

In this paper, we will study the $(\mu+1)$-Evolutionary Algorithm (\moea) with \emph{standard bit mutation} and \emph{elitist selection}. So for offspring generation, a parent $x$ is chosen uniformly at random from $P^k$, and the offspring is generated by flipping each bit of $x$ independently with probability $c/n$, where $c$ is the \emph{mutation parameter}. For selection, we simply select the $\mu$ individuals with largest $f^k$-values to form population $P^{k+1}$. 
%For crossover, two individuals $x,x'$ from $P^k$ are chosen uniformly at random (with repetition), and the offspring copies each position from either $x$ or $x'$ with probability $1/2$ each, independently for each position.
%We will focus on the \emph{$(\mu+1)$ Evolutionary Algorithm} \moea and the \emph{$(\mu+1)$ Genetic Algorithm} \moga. Both have the mutation parameter $c$ as parameter. For the \moea, the offspring is generated by mutation. For the \moga, in each generation the offspring is generated either by mutation or by crossover, each with probability $1/2$.

For the dynamic binary value function \dynbv, for each $k\geq 0$ a uniformly random permutation $\pi^{k}:\{1,\ldots,n\} \to \{1,\ldots,n\}$ is drawn, and the fitness function for generation $k$ is then given by $f^{k}(x) = \sum_{i=1}^n 2^{n-i}x_{\pi^{k}(i)}$.

%\jl{todo: replace the pseudocode by a generic description of a dynamic optimization}

%The algorithm, in this dynamic environment, proceeds in each new time step as follows. First it generates $\lambda$ offspring, employing only mutation or even crossover in case of a genetic algorithm, then proceeds to select the fittest $\mu$ individuals out of the $\mu+\lambda$ according to the \dynbv function in the new iteration. 

\subsection{Notation and Setup}\label{sec:setup}
Throughout the paper, we will assume that the population size $\mu$ and the mutation parameter $c$ are constants, whereas $n$ tends to $\infty$. We use the expression ``with high probability'' or whp for events $\mathcal E_n$ such that $\Pr(\mathcal E_n) \to 1$ for $n\to \infty$. We write $x = O(y)$, where $x$ and $y$ may depend on $n$, if there is $C>0$ such that $|x| \leq Cy$ for sufficiently large $n$. Note that we take the absolute value of $x$. The statement $x = O(y)$ does not imply that $x$ must be positive. Consequently, if we write an expression like $\Delta = \eps + O(\eps^2)$ then we mean that there is a constant $C>0$ such that $\eps - C\eps^2 \leq \Delta \leq \eps + C\eps^2$ for sufficiently large $n$. It does not imply anything about the sign of the error term $O(\eps^2)$. We will sometimes use minus signs or ``$\pm$'' to ease the flow of reading, e.g., we write $1-o(1)$ for probabilities. But this is a cosmetic decision, and is equivalent to $1+o(1)$.

For two bit-strings $x,y\in\{0,1\}^n$, we say that $x$ \emph{dominates} $y$ if $x_i \geq y_i$ for all $i\in\{1..n\}$.

Our main tool will be drift theory. In order to apply this, we need to identify states that we can adequately describe by a single real value. Following the approach in~\cite{lengler2019general} and~\cite{lengler2020large}, we call a population \emph{degenerate} if it consists of $\mu$ copies of the same individual. If the algorithm is in a degenerate population, we will study how the \emph{next degenerate population} looks like, so we define
\begin{align}\label{eq:def_Phit}
\Phi^t := \{\# \text{ of zero-bits in an individual in the $t$-th degenerate population}\}.
\end{align}
Our main object of study will be the \emph{degenerate population drift} (or simply \emph{drift} if the context is clear). For $0\leq \eps \leq 1$, it is defined as
\begin{align}\label{eq:def_drift}
\Delta(\eps) := \Delta^t(\eps):= \EX[\Phi^t-\Phi^{t+1} \mid \Phi^t = \lfloor \eps n\rfloor ].
\end{align}
The expression is independent of $t$ since the considered algorithms are time-homogeneous. If we want to stress that $\Delta(\eps)$ depends on the parameters $\mu$ and $c$, we also write $\Delta(\mu, c, \eps)$. 
%, and if we want to include the algorithm that we consider, we write $\DeltaEA(\mu, c, \eps)$ and $\DeltaGA(\mu, c, \eps)$ for the drift of the \moea and the \moga, respectively. 
Note that the number of generations to reach the $(t+1)$-st degenerate population is itself a random variable. So the number of generations to go from $\Phi^t$ to $\Phi^{t+1}$ is random. As in~\cite{lengler2019general}, its expectation is $O(1)$ if $\mu$ and $c$ are constants, and it has an exponentially decaying tail bound, see Lemma~\ref{lem:tailbound} below. In particular, the probability that during the transition from one degenerate population to another the same bit is touched by two different mutations is $O(\eps^2)$, and likewise the contribution of this case to the drift is $O(\eps^2)$, as we will prove formally in Lemma~\ref{lem:singlezerobit}.

We remark that we do not assume constant $\eps$, i.e., $\eps = \eps(n)$ may depend on $n$. For our main result we \emph{will} choose $\eps$ to be a sufficiently small constant, but we need to choose it such that we can determine the sign of terms like $\eps\pm O(\eps^2)\pm o(1)$. Note that this is possible: if $\eps>0$ is a sufficiently small constant then $\eps \pm O(\eps^2) \geq \eps/2$, and if afterwards we choose $n$ to be sufficiently large then the $o(1)$ term is at most $\eps/4$. Since this is subtle point, we will not treat $\eps$ as a constant. In particular, all $O$-notation is with respect to $n\to\infty$, and does not hide dependencies on $\eps$. This is why we have to keep error terms like $O(\eps^2)$ and $o(1)$ separate.

To compute the degenerate population drift, we will frequently need to compute the expected change of the potential provided that we visit an intermediate state $S$. Here, a state $S$ is simply given by a population of $\mu$ search points. We will call this change the \emph{drift from state $S$}, and denote it by $\Delta(S,\eps)$. Formally, if $\mathcal{E}(S,t)$ is the event that the algorithm visits state $S$ between the $t$-th and $(t+1)$-st degenerate population,
\begin{align}\label{eq:def_contribution}
\Delta(S,\eps) := \EX[\Phi^t-\Phi^{t+1} \mid \Phi^t = \eps n \text{ and } \mathcal{E}(S,t)].
\end{align}
This term is closely related to the \emph{contribution to the degenerate population drift from state $S$}, which also contains the probability to reach $S$ as a factor: 
\begin{align}\label{eq:def_contribution2}
\Deltacon(S,\eps) := \Pr[\mathcal{E}(S,t) \mid \Phi^t = \eps n] \cdot \Delta(S,\eps).
\end{align}

We will study \dynbv around the optimum, i.e., we consider any $\eps = \eps(n)\to0$ for $n\to \infty$, and we compute the asymptotic expansion of $\Delta(\eps)$ for $n\to\infty$. As we will see, the drift is of the form $\Delta(\eps) = a\eps + O(\eps^2) + o(1)$ for some constant $a \in \RR$. 
Analogously to~\cite{lengler2019general} and~\cite{lengler2019exponential}, if~$a$ is \emph{positive} (multiplicative drift), then the algorithm starting with at most $\eps_0 n$ zero-bits for some suitable constant $\eps_0$ whp needs $O(n\log n)$ generations to find the optimum. On the other hand, if~$a$ is \emph{negative} (negative drift/updrift), then whp the algorithm needs exponentially many generations to find the optimum (regardless of whether it is initialized randomly or with $\eps_0 n$ zero-bits). These two cases are typical. There is no constant term in the drift since for a degenerate population $P^k$ we have $P^{k+1} = P^k$ with probability $1-O(\eps)$. This happens whenever mutation does not touch any zero-bit, since then the offspring is rejected.\footnote{Here and later we use the convention that if an offspring is identical to the parent, and they have lowest fitness in the population, then the offspring is rejected. Since the outcome of ejecting offspring or parent is the same, this convention does not change the course of the algorithm.}

We will prove that, as long as we are only interested in the first order expansion (i.e., in a results of the form $a\eps + O(\eps^2) + o(1)$), we may assume that between two degenerate populations, the mutation operators always flip different bits. In this case, we use the following naming convention for search points. The individuals of the $t$-th degenerate population are all called $x^0$. We call other individuals $x^{(m_1-m_2)}$, where $m_1$ stands for the extra number of ones and $m_2$ for the extra number of zeros compared to $x^0$. Hence, if $x^0$ has $m$ zero-bits then $x^{(m_1-m_2)}$ has $m+m_2-m_1$ zero-bits.  Following the same convention, we will denote by $X_k^z$ a set of $k$ copies of $x^z$, where the string $z$ may be $0$ or $(m_1-m_2)$. In particular,  $X^0_{\mu}$ denotes the $t$-th degenerate population. 

\subsection{Duration Between Degenerate Populations}
We formalize the assertions in Section~\ref{sec:setup} that the number of steps between two degenerate populations satisfies exponential tail bounds, and that it is unlikely to touch a bit by two different mutations as we transition from one degenerate population to the next. We give proofs for completeness, but similar statements are well-known in the literature.

\begin{lemma}\label{lem:tailbound}%old Lemma 1
For all constant $\mu,c$ there is a constant $a >0$ such that the following holds for the \moea with mutation parameter $c$ in any population $X$ on \dynbv. Let $K$ be the number of generations until the algorithm reaches the next degenerate population. Then for all $k\in \NN_0$,
\[
\Pr(K \geq k \cdot \mu) \leq e^{-a \cdot k}.
\] 
\end{lemma}
\begin{proof}
Let $x^0 \in X$ be the individual with the least number of zero-bits, and let $\hat{p}$ be the probability to degenerate in the next $\mu$ steps. Clearly, $\hat{p}$ is at least the probability that in each step we copy $x^0$ and accept it into the population. The probability of selecting $x^0$ and mutating no bits is at least $\tfrac{1}{\mu} (1-\tfrac{c}{n})^{n} \xrightarrow{n\rightarrow{\infty}} e^{-c}/\mu$. Since $x^0$ is the individual with the least number of zeros in the population, the probability that it is not worst in the population and thus all copies are kept is at least $1/2$: any other individual $y \neq x^{0}$ will have at least as many zeros as $x$ and therefore will be ranked lower than $x^0$ with probability at least $1/2$.
Therefore, for sufficiently large $n$,
 \[
 \hat{p} \geq \Big(\frac{e^{-c}}{4 \mu}\Big)^{\mu - 1}.
 \]
This bound works for any starting population. So if we don't degenerate in the first $\mu$ steps of the algorithm we again have probability $\hat{p}$ to degenerate in the successive $\mu$ steps, and so on. Therefore, we can simply bound the probability not to degenerate in the first $k\cdot\mu$ steps by $(1-\hat{p})^{k} \leq e^{-\hat{p} \cdot k}$, where the last step uses the Bernoulli inequality $(1+x) \leq e^{x} \thinspace \thinspace \forall x \in \RR$.\qed
\end{proof}

\begin{lemma}\label{lem:singlezerobit}%old Lemma 2
Consider the \moea with mutation parameter $c$ on\dynbv. Let $X^t$ and $X^{t+1}$ denote the $t$-th and $(t+1)$-st degenerate population respec\-tively. Let $\eps >0$, and let $X$ be a degenerate population with at most $\eps n$ zero-bits. 
\begin{enumerate}[(a)]
\item Let $\mathcal E_2$ be the event that the mutations during the transition from $X^t$ to $X^{t+1}$ flip at least two zero-bits. Then $\Pr[\mathcal E_2 \mid X^t = X] = O(\eps^2)$. Moreover, the contribution of this case to the degenerate population drift $\Delta$ is
\[
\Delta^*(\eps) := \Pr[\mathcal E_2 \mid X^t = X] \cdot \E[\Delta^t(\eps) \mid \mathcal E_2 \wedge X^t = X] = O(\eps^2).
\] 
\item Let $S$ be any non-degenerate state such that there is at most one position which is a one-bit in some individuals in $S$, but a zero-bit in $X$. Let $\mathcal E(S,t)$ be the event that state $S$ is visited during the transition from $X^t$ to $X^{t+1}$, and let $\mathcal E_1$ be the event that a zero-bit is flipped in the transition from $S$ to $X^{t+1}$. Then $\Pr[\mathcal E_1 \mid \mathcal E(S,t) \wedge X^t = X] = O(\eps)$, and the contribution to $\Delta(S,\eps)$ is
\begin{equation}\label{eq:contribution2}
\Delta^*(S, \eps) := \Pr[\mathcal E_1 \mid \mathcal E(S,t) \wedge X^t = X] \cdot \E[\Delta^t(\eps) \mid \mathcal E_1 \wedge \mathcal E(S,t) \wedge X^t = X] = O(\eps).
\end{equation}
%if $\Pr[\mathcal E(S,t)] = O(\eps)$ then 
The contribution of the case $\mathcal E(S,t) \wedge \mathcal E_1$ to the degenerate population drift is 
\[
\Deltacon^*(S,\eps) := \Pr[\mathcal E(S,t)]\cdot\Delta^*(S,\eps) = O(\eps^2).
\]
\end{enumerate}
In both parts, the hidden constants do not depend on $S$ and $X$.
\end{lemma}
%We remark that the condition $\Pr[\mathcal E(S,t)] = O(\eps)$ is always satisfied in our analyses since the probability to leave a degenerate state is $O(\eps)$.
\begin{proof}
\emph{(a).} If the offspring in the first iteration is not accepted into the population or is identical to $x^0$, then the population is immediately degenerate again, and there is nothing to show. So let us consider the case that the offspring is different and is accepted into the population. Then the mutation needed to flip at least one zero-bit, since otherwise the offspring is dominated by $x^0$ and rejected. Thus the probability of this case is $O(\eps)$. Moreover, the probability of flipping at least two zero-bits in this mutation is $O(\eps^2)$, so we may assume that \emph{exactly one} zero-bit is flipped. 

Let $k$ be such that the number of subsequent iterations until the next degenerate population is in $[\mu k, \mu (k+1)]$. Conditioning on $k$, we have at most $\mathcal{O}(k)$ mutations until the population degenerates, each with a probability of at most $\eps c$ of flipping a zero-bit. Using Lemma~\ref{lem:tailbound} we have for some $a>0$,
\[\Pr[\mathcal E_2 \mid X^t = X] \leq \mathcal{O}(\epsilon) \cdot \sum_{k=1}^{\infty} e^{-a\cdot k} \cdot \mathcal{O}(\epsilon \cdot k) = \mathcal{O}(\epsilon^2).\]
To bound the contribution to the drift of the case where an additional zero-bit flip happens, we will bound the expectation of $|\Phi^t - \Phi^{t+1}|$, conditioned on being in this case. That difference is at most the number of bit flips until the next degenerate population, which is $\mathcal{O}(k)$ in expectation. Again, summing over all possible $k$ and using Lemma~\ref{lem:tailbound} we get
\[
\Delta^*(\eps) \leq \mathcal{O}(\epsilon )\cdot \sum_{k = 1}^{\infty} \mathcal{O}(\epsilon \cdot k) \cdot e^{-a \cdot {k}} \cdot \mathcal{O}(k) = \mathcal{O}\left(\epsilon^2\sum_{k = 1}^{\infty} k^2e^{-a \cdot {k}}\right) = \mathcal{O}(\eps^2). \]
\emph{(b).} The proof is identical to (a), except that the probability $O(\eps)$ of leaving $X^t$ is treated separately. We omit the details.\qed
\end{proof}

Before we start to analyze the algorithms, we prove a helpful lemma to classify how the population can degenerate if no zero-bit is flipped. As we have explained in Section~\ref{sec:setup} (and made formal in Lemma~\ref{lem:tailbound} and Lemma~\ref{lem:singlezerobit}), this assumption holds with high probability.  In this case, the population degenerates to copies of an individual which is not dominated by any other search point. 

\begin{lemma}\label{lem:domination}%old Lemma 0
Consider the \moea in a non-degenerate population $X$. Let $x_1,x_2,...,x_k$ be search points in $X$ that dominate all the rest of the population. Then either at least one zero-bit is flipped until the next degenerate population, or the next degenerate population consists of copies of one of the search points $x_1,x_2,...,x_k$. 
\end{lemma}
\begin{proof}
%Note that the transitivity holds for the domination property, in particular, if $x$ dominates $y$ and $y$ dominates $z$, we have that $x$ dominates $z$. 
Assume that, starting from X, the algorithm doesn't flip any additional zero-bits. We start by inductively showing that for all subsequent time steps, every individual in the population is still dominated by one of the search points $x_1,x_2,...,x_k$. Suppose, for the sake of contradiction, that eventually there are individuals which are not dominated by any of the search points in $\{x_1,x_2,...,x_k\}$, and let  $x^*$ be the first such individual. Since we assumed that the algorithm doesn't flip any additional zero-bits, $x^*$ must have been generated by mutating an individual $\Bar{x}$ and only flipping one-bits. So $\Bar{x}$ dominates $x^*$. On the other hand, $\Bar{x}$ is dominated by one of the search points $x_1,x_2,...,x_k$ by our choice of $x^*$. This is a contradiction since domination is transitive. Therefore, using transitivity, the algorithm will not generate any individual that is not dominated by any search point in $\{x_1,x_2,...,x_k\}$. Furthermore, the population will never degenerate to any other individual $\Tilde{x} \notin \{x_1,x_2,...,x_k\}$. In fact, let $x_i$ be the search point in $\{x_1,x_2,...,x_k\}$ that dominates $\Tilde{x}$. We have that $f(\Tilde{x}) < f(x_i)$ in all iterations and for all permutations, therefore $x_i$ will never be discarded before $\Tilde{x}$, which concludes the proof.\qed
\end{proof}

%%%%%%%%%%%%%%%%%%%%% DRIFT %%%%%%%%%%%%%%%%%%

\section{Analysis of the Degenerate Population Drift}
%\subsection{Results for the \moea}
In this section, we will find a lower bound for the drift $\Delta(\eps) = \Delta(\mu,c,\eps)$ of the \moea close to the optimum, when $n \rightarrow{\infty}$. %In particular, we begin by revisiting the \toea studied in \cite{lengler2020large} and finding a closed formula for the drift. The reasoning used to find it will later help us prove the desired lower bounds for the drift of the \moea.
The main result of this section will be the following. 
\begin{theorem}\label{thm:DeltaEA}%old Proposition 2
For all $c>0$ there exist $\delta, \epsilon_0 >0$ such that for all $\epsilon \le \epsilon_0$ and $\mu \geq \mu_0:= e^c+2$, if $n$ is sufficiently large,
\[\DeltaEA(c, \mu, \eps) \geq \delta \cdot \epsilon.\]
\end{theorem}

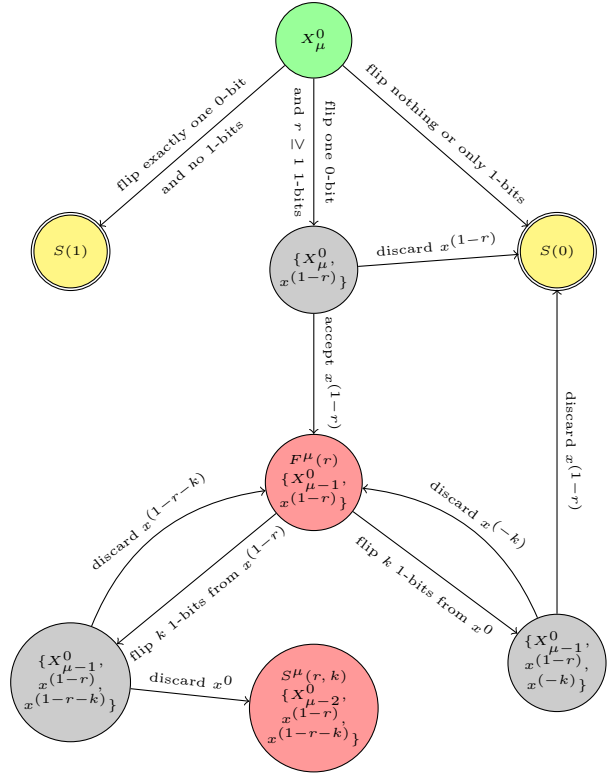
\begin{wrapfigure}{r}{0.6\textwidth}%\includepdf[scale=0.8, pagecommand={\null\vfill\captionof{figure}{Insert new caption here}}]{(mu+1)EA_Diagram.pdf}\label{fig:transitionmoea}
\begin{tikzpicture}[main node/.style={circle,draw,font=\tiny, minimum size = 1cm, inner sep = 0 cm},scale = 0.8]
  \node[fill=gray!40] at (0,.6) [main node] (C) {$\begin{array}{c}
  \{X_\mu^0,\\x^{(1-r)}\}
  \end{array}$};
  \node[fill=green!40] at (0,4.5) [main node] (A) {$X_\mu^0$};
  \node[fill=yellow!60, double] at (4,1) [main node] (D) {$S(0)$};
  \node[fill=yellow!60, double] at (-4,1) [main node] (B) {$S(1)$};
  \node[fill=red!40] at (0,-3) [main node] (F) {$\begin{array}{c}
  F^\mu(r)\\\{X_{\mu-1}^0,\\x^{(1-r)}\}
  \end{array}$};
  \node[fill=gray!40] at (-4,-6.5) [main node] (E) {$\begin{array}{c}
  \{X_{\mu-1}^0,\\x^{(1-r)},\\x^{(1-r-k)}\}
  \end{array}$};
  \node[fill=gray!40] at (4,-6) [main node] (G) {$\begin{array}{c}
  \{X_{\mu-1}^0,\\x^{(1-r)},\\x^{(-k)}\}
  \end{array}$};
  \node[fill=red!40] at (0,-7) [main node] (H) {$\begin{array}{c}
  S^\mu(r,k)\\\{X_{\mu-2}^0,\\x^{(1-r)},\\x^{(1-r-k)}\}
  \end{array}$};
%  \node[fill=yellow!60, double] at (4,-9) [main node] (I) {$S(0)$};
  
%  \draw[->] (-2,6) -- (A);
  \draw[->] (A) -- node [sloped, midway,above ] {\tiny flip one $0$-bit} node [sloped, midway,below ] {\tiny and $r \geq 1$ $1$-bits} (C);
  \draw[->] (A) -- node [sloped, midway,above ] {\tiny flip exactly one $0$-bit} node [sloped, midway,below ] {\tiny and no $1$-bits} (B);
  \draw[->] (A) -- node [sloped, midway,above ] {\tiny flip nothing or only $1$-bits} (D);
  \draw[->] (C) -- node [sloped, midway,above ] {\tiny discard $x^{(1-r)}$} (D);
  \draw[->] (C) -- node [sloped, midway,above ] {\tiny accept $x^{(1-r)}$} (F);
  \draw[->] (F) -- node [sloped, midway,below ] {\tiny flip $k$ $1$-bits from $x^{(1-r)}$} (E);
  \draw[->] (E) to[bend left] node [sloped, midway,above] {\tiny discard $x^{(1-r-k)}$} (F);
  \draw[->] (F) -- node [sloped, midway,below ] {{\tiny flip $k$ $1$-bits from $x^{0}$}} (G);
  \draw[->] (G) to[bend right] node [sloped, midway,above] {\tiny discard $x^{(-k)}$} (F);  
  \draw[->] (E) -- node [sloped, midway,above ] {\tiny discard $x^{0}$} (H);
%  \draw[->] (G) -- node [sloped, midway,above ] {\tiny reject $x^{(1-r)}$} (I);
  \draw[->] (G) -- node [sloped, midway,above ] {\tiny discard $x^{(1-r)}$} (D);
\end{tikzpicture}
\caption{State Diagram for the \moea}\label{fig:transitionmoea}
\end{wrapfigure}

Lemma~\ref{lem:domination} allows us to describe the transition from one degenerate population to the next by a relatively simple Markov chain, provided that at most one zero-bit is flipped during the transition. This zero-bit needs to be flipped in order to leave the starting state, so we assume for this chain that no zero-bit is flipped afterwards. This assumptions is justified by Lemma~\ref{lem:singlezerobit}. The Markov chain (or rather, a part of it) is shown in Figure~\ref{fig:transitionmoea}. The starting state, which is a degenerate population, is depicted in green. The yellow states $S(k)$ represent degenerate populations where the number of one-bits is exactly $k$ larger than that of the starting state, so $\Phi^{t+1}-\Phi^t = k$. In later diagrams, we will also see negative values of $k$. We have included intermediate states depicted in gray, in which an offspring has been created, but selection has not yet taken place. In other words, the gray states have $\mu+1$ search points, and it still needs to be decided which of them should be discarded from the population. As we will see in the analysis, it is quite helpful to separate offspring creation from this selection step. The remaining states are depicted in red. Note that we have only drawn part of the Markov chain since from the bottom-most state $S^{\mu}(r,k)$, we have not drawn outgoing arrows or states.

Moreover, the states of the Markov chain do not correspond one-to-one to the generations: we omit intermediate states where Lemma~\ref{lem:domination} allows us to do that. For example, following the first arrow to the left we reach a state in which one individual $x^{(1)}$ (the offspring) dominates all other individuals. By Lemma~\ref{lem:domination}, such a situation must degenerate into $\mu$ copies of of $x^{(1)}$, so we immediately mark this state as a degenerate state with $\Phi^{t+1}-\Phi^t = 1$.

%Let us start by defining:
%\begin{align*}
%    \Delta_1(\mu, c, x) := &\text{Drift of \moea close to optimum for the BinVal function,} \\ 
%    & \text{starting with a degenerate population of copies of x}.
%\end{align*}
%
%We remind the reader that by drift we intend $\EX[\Phi_t - \Phi_{t+1} \mid \Phi_t = \epsilon \cdot n]$, where $\Phi_t$ is the number of zeros of the individuals in the $t$-th degenerate population. Also, refer to $\text{Section 2.1}$ for the term $\text{close to the optimum}$. 

The key step will be to give a lower bound for the contribution to the drift from state $F^{\mu}(r)$. Once we have a bound on this, it is straightforward to compute a bound on the degenerate population drift. Before we turn to the computations, we first introduce a bit more notation.
\begin{definition}\label{def:B_i}
Consider the \moea in state $F^{\mu}(r)$ in generation $k-1$. We re-sort the $n$ positions of the search points descendingly according to the next fitness function $f^{k}$. So by the ``first'' position we refer to the position which has highest weight according to $f^k$, and the $j$-th bit of a bitstring $z$ is given by $z_i$ such that $\pi^{k}(i) = j$. Then we define: %So we say that the $j$-th bit of a bitstring $y$ is given by $y_i$ such that $p^{(k)}(i) = j$. We apply the ordering given in iteration $k$ for the definitions of $B_{z}^{k}$ and $B_0^k$.
\begin{itemize}
    \item $B_{z}^{k} := \text{position of the first zero-bit in $z$}$;
    \item $B_0^k := \text{position of the first flipped bit in the $k$-th mutation}$;
    \item $z_1^{k} := \argmin\{f^{k}(z) \mid z \in\{X^0_{\mu-1}, X^{(1-r)}_1\}\};$
    \item $z_2^{k} := \argmax\{f^{k}(z) \mid z \in \{X^0_{\mu-1}, X^{(1-r)}_1\}\}$.
\end{itemize} 
In particular, the search point to be discarded in generation $k$ is either $z_1^{k}$ or it is the offspring generated by the $k$-th mutation. We define $B_0^k$ to be $\infty$ if no bits are flipped in the $k$-th mutation.
\end{definition}

Now we are ready to bound the drift of state $F^{\mu}(r)$. We remark that the statement for $\mu=2$ was also proven in~\cite{lengler2020fullversion}, but the proof there was much longer and more involved, since it did not make use of the hidden symmetry of the selection process that we will use below.

\begin{lemma}\label{lem:Fr}%old Lemma 4
Consider the \moea on the \dynbv function in the state $F^{\mu}(r)$ for some $r\geq 1$, and let $\eps >0$. Then the drift from $F^{\mu}(r)$ is
\[
\Delta(F^{\mu}(r),\eps) \geq \frac{1-r}{1+(\mu-1) \cdot r} + \mathcal{O}(\epsilon).
\] 
For $\mu=2$ the bound holds with equality, i.e. $\Delta(F^{2}(r),\eps) = (1-r)/(1+ r) + \mathcal{O}(\epsilon).$
\end{lemma}

\begin{proof}
Let us assume that the algorithm will not flip an additional zero-bit through mutation before it reaches the next degenerate population. In fact, the contribution to the drift in case it does flip another zero-bit can be summarized by $\mathcal{O}(\epsilon)$ due to Equation~\eqref{eq:contribution2} in Lemma~\ref{lem:singlezerobit}. So from now on, we assume that the algorithm doesn't flip an additional zero-bit until it reaches the next degenerate population.

The idea is to follow the Markov chain as shown in Figure~\ref{fig:transitionmoea}. We will compute the conditional probabilities of reaching different states from $F^{\mu}(r)$, conditional on actually leaving $F^{\mu}(r)$. More precisely, we will condition on the event that an offspring $\Bar{x}$ is generated and accepted into the population. %We consider this event as ``leaving'' the state $F^{\mu}(r)$, even though the new state has the same population. 

Recall that $F^{\mu}(r)$ corresponds to the population of $\{X^0_{\mu-1}, X^{(1-r)}_1\}$, i.e., $\mu-1$ copies of $x^0$ and one copy of $x^{(1-r)}$. So if the offspring is accepted, one of these search points must be ejected from the population. Let us first consider the case that $x^{(1-r)}$ is ejected from the population. Then the population is dominated by $x^0$ afterwards, and will degenerate into $X_\mu^{0}$ again by Lemma~\ref{lem:domination}. The other case is that one of the $x^0$ individuals is ejected, which is described by state $S^{\mu}(r,k)$. It is complicated to compute the contribution of this state precisely, but by Lemma~\ref{lem:domination} we know that this population will degenerate either to copies of $x^0$ or of $x^{(1-r)}$. For $\mu =2$, only the second case is possible, since there are no copies of $x^0$ left in~$S^{\mu}(r,k)$. In any case, we can use the pessimistic bound $\Delta(S^{\mu}(r,k),\eps) \geq (1-r)$ for the drift of $S^{\mu}(r,k)$, with equality for $\mu =2$.\footnote{The notation is slightly imprecise here, since we condition on the event that no further zero-bit is flipped, which is not reflected in the notation. But as argued above, this only adds an additive $O(\eps)$ error term to the final result.} Summarizing, once a new offspring is accepted, if a copy of $x^0$ is discarded we get a drift of at most $1-r$ and if $x^{(1-r)}$ is discarded we get a drift of $0$. It only remains to compute the conditional probabilities with which these cases occur.

%revolves around the observation that it is easy to compute the drift of a population if we assume that the population stays the same until it degenerates. More precisely, we will assume that there is some fixed iteration of the algorithm where a new offspring is accepted in the population and from there it is clear where the latter will degenerate by using Lemma~\ref{lem:domination}. To start addressing the assumption of the population degenerating after an individual is accepted we need to note how the population looks in state $F^{\mu}(r)$ , namely we have a population of $\{X^0_{\mu-1}, X^{(1-r)}_1\}$. There are two cases for the accepted offspring $\Bar{x}$. In the first case it is dominated by $x^0$, hence conditioning on it being accepted we always end in a population of $\{X^0_{\mu-1}, \Bar{x} \}$, where we can use Lemma~\ref{lem:domination}. In the second case it is dominated by $x^{(1-r)}$, therefore ending in a population of $\{X^0_{\mu-2}, X^{(1-r)}_1, \Bar{x} \}$, state $S^{\mu}(r,k)$ in the figure \sr{need to insert figures}. We can use use the pessimistic bound $\Delta(S^{\mu}(r,k),\eps) \geq (1-r)$ for the contribution to the drift once the population reaches state $S^{\mu}(r,k)$, which follows from Lemma~\ref{lem:domination}. Therefore, once a new offspring is accepted, if a copy of $x^0$ is discarded we get a contribution to the drift of $1-r$ and if $x^{(1-r)}$ is discarded we get a contribution of 0.

To compute the probabilities is not straightforward, but we can use a rather surprising symmetry, using the terminology from Definition~\ref{def:B_i}. Assume that the algorithm is in iteration $k$. We make the following observation: an offspring is accepted if and only if it is mutated from $z_2^{k}$ and $B_0^k > B_{\min}^k:=\min\{B_{x^0}^k,B_{x^{(1-r)}}^k\}$.
%Therefore the population stays the same, up until in some iteration $\ell \geq k$ an offspring of $z_2^{\ell}$ with $B_0^\ell < B_{\min}^\ell$ is accepted. 
Hence, we need to compute the probability
\[
\hat{p} := \Pr\big(f^k(x^{(1-r)}) \geq f^k(x^0) \mid \text{\{mutated $z_2^{k}$\}} \wedge \{B_0^k > B_{\min}^k\} \big),
\]
since then we can bound $\DeltaEA(F^{\mu}(r),\eps) \geq (1-r)\hat{p} + \mathcal{O}(\epsilon)$ by Lemma~\ref{lem:singlezerobit}. For $\mu=2$, this lower bound is an equality. 

%Then the drift is given by $\Delta(F^{\mu}(r) \geq $ \sr{continue here}
%
%We note that if $z_2^{\ell} = x^0$, then $x^{(1-r)}$ would get rejected
%and we will degenerate to copies of $x^0$, yielding no contribution to the drift. Therefore, it is sufficient to compute the probability 
%\[\hat{p} := \Pr\left(f(x^{(1-r)}) \geq f(x^0) \mid \text{\{mutated $z_2^{(k^{*})}$\}} \wedge \{B_0 > min\{B_x,B_y\}\} \right).\]
Clearly, the events $\{f^k(x^{(1-r)}) \geq f^k(x^0)\}$ and $\{B_0 > B_{\min}^k\}$ are independent. We emphasize that this is a rather subtle symmetry of the selection process. Using conditional probability, $\hat{p}$ simplifies to:
\begin{align}\label{eq:hatp}
\hat{p} = \frac{\Pr\big(f(x^{(1-r)}) \geq f(x^0) \wedge \text{\{mutated $z_2^{k}$\}}\big)}{\Pr\big(\text{\{mutated $z_2^{k}$\}} \big)}. 
\end{align}
To compute the remaining probabilities, we remind the reader that $x^{(1-r)}$ has exactly $r$ more zero-bits and 1 more one-bit, than $x^{0}$. Hence, in order to compare them, we only need to look at the relative positions of these $r+1$ bits in which they differ. In particular, $x^{(1-r)} = z_2^{k}$ holds if and only if the permutation $\pi^{k}$ places the one-bit from $x^{(r-1)}$ before the $r$ one-bits of $x^{0}$, and this happens with probability $1/(r+1)$. Moreover, recall that there are $\mu-1$ copies of $x^0$ and only one $x^{(1-r)}$, so the probability of picking them as parents is $(\mu-1)/\mu$ and $1/\mu$, respectively. Therefore, by using the law of total probability,
\begin{align*}
    \Pr\big(\text{\{mutated $z_2^{k}$\}} \big) & = \Pr\big(\text{\{mutated $z_2^{k}$\}} \mid x^{(1-r)} = z_2^{k}\big) \cdot \Pr\big(x^{(1-r)} = z_2^{k}\big) \\
    & \quad + \Pr\big(\text{\{mutated $z_2^{k}$\}} \mid x^{0} = z_2^{k}\big) \cdot \Pr\big(x^{0} = z_2^{k}\big)\\
    & = \frac{1}{\mu} \cdot \frac{1}{r+1} + \frac{\mu -1}{\mu} \cdot \frac{r}{r+1}
\end{align*}
Plugging this into~\eqref{eq:hatp} yields
%To compute the probability $\Pr(x^{0} = z_2^{k})$, we remind the reader that $x^{(1-r)}$ has exactly $r$ more zero-bits and 1 more one-bit, than $x^{0}$. Hence, in order to compare them, we only need to look at the relative positions of these $r+1$ bits in which they differ. In particular, $x^{(1-r)} = z_2^{k}$ holds if and only if the permutation $\pi^{k}$ places the one-bit from $x^{(r-1)}$ before the $r$ one-bits of $x^{0}$, and this happens with probability $1/(r+1)$. Therefore,
%\[\Pr\big(\text{\{mutated $z_2^{k}$\}} \big) = \frac{1}{\mu} \cdot \frac{1}{r+1} + \frac{\mu -1}{\mu} \cdot \frac{r}{r+1}.\]
%This allows us to write:
\[\hat{p} = \left(\frac{1}{r+1} \cdot \frac{1}{\mu}\right)\bigg/\left(\frac{1}{\mu} \cdot \frac{1}{r+1} + \frac{\mu -1}{\mu} \cdot \frac{r}{r+1}\right) = \frac{1}{1+(\mu-1) r}.\]
Together with Lemma~\ref{lem:singlezerobit} and the lower bound $\Delta(F^{\mu}(r),\eps) \geq (1-r)\hat p + O(\eps)$, this concludes the proof. For $\mu =2$, the lower bound is an equality. \qed
\end{proof}  
%
%\begin{remark}\label{rem:toea}
%The lower bound $\Delta(S^{\mu}(r,k),\eps) \geq (1-r)$ used in the proof of Lemma~\ref{lem:Fr} is actually an equality if $\mu=2$, since then both individuals in the state $S^{\mu}(r,k)$ are dominated by $x^{(1-r)}$, and thus the population degenerates to $x^{(1-r)}$ by Lemma~\ref{lem:domination}. Hence, Lemma~\ref{lem:Fr} holds with equality for $\mu=2$. The statement for $\mu=2$ was also proven in~\cite{lengler2020fullversion}, but the proof there was quite long and involved, since it did not make use of the hidden symmetry of the selection process.
%\end{remark}
Now we are ready to bound the degenerate population drift and prove Theorem~\ref{thm:DeltaEA}. 
\begin{proof}[of Theorem~\ref{thm:DeltaEA}]
To prove this theorem, we refer to Figure~\ref{fig:transitionmoea}. By Lemma~\ref{lem:singlezerobit}, the contribution of all states that involve flipping more than one zero-bit is $O(\eps^2)$. If we flip no zero-bits at all, then the population degenerates to $X_\mu^0$ again, which contributes zero to the drift. So we only need to consider the case where we flip exactly one zero-bit in the transition from the $t$-th to the $(t+1)$-st degenerate population. This zero-bit needs to be flipped in the first mutation, since otherwise the population does not change. We denote by $p_r$ the probability to flip exactly one zero-bit and $r$ one-bits in $x^0$, thus obtaining $x^{(1-r)}$. If $f^k(x^{(1-r)}) > f^k(x^0)$ then $x^{(1-r)}$ is accepted into the population and we reach state $F^{\mu}(r)$. This happens if and only if among the $r+1$ bits in which $x^{(1-r)}$ and $x^{0}$ differ, the zero-bit of $x^0$ is the most relevant one. So $\Pr[f^k(x^{(1-r)}) > f^k(x^0)] = 1/(r+1)$ Finally, by Lemma~\ref{lem:Fr}, the drift from $F^{\mu}(r)$ is at least $-(r-1)/(1+(\mu-1) r) + \mathcal{O}(\epsilon)$. Summarizing all this into a single formula, we obtain
\begin{align}\label{eq:DeltaEA1}
    \DeltaEA(\eps) & \geq \mathcal{O}(\epsilon^2) + p_0 +  \sum_{r=1}^{(1-\epsilon)  n} p_r \cdot \left[\Pr[f^k(x^{(1-r)}) >f^k(x^0)] \cdot \DeltaEA(F^{\mu}(r),\eps)\right] \nonumber\\
& \geq \mathcal{O}(\epsilon^2) + p_0 -  \sum_{r=1}^{(1-\epsilon) n} p_r \cdot \frac{1}{r+1}\cdot \left(\frac{r-1}{1+(\mu -1) r} + \mathcal O(\eps)\right).
\end{align}
For $p_r$, we use the following standard estimate, which holds for all $r = o(n)$. 
\[
p_r = (1 + o(1)) \cdot c^{r+1}/r! \cdot e^{-c} \cdot \epsilon \cdot (1-\epsilon)^r.
\]
The summands for $r = \Omega(n)$ (or $r = \omega(1)$, actually) in \eqref{eq:DeltaEA1} are negligible since $p_r$ decays exponentially in $r$. Plugging $p_0$ and $p_r$ into~\eqref{eq:DeltaEA1}, we obtain
\begin{align}\label{eq:DeltaEA2}
    \DeltaEA(\eps) &\geq \mathcal{O}(\epsilon^2) + (1 + o(1))  \epsilon  c  e^{-c}  \bigg[ 1 - \sum_{r=1}^{(1-\epsilon) n} \underbrace{\frac{c^r}{(r+1)!} \cdot \frac{(1-\epsilon)^r  (r-1)}{ (1+(\mu - 1)  r)}}_{=: f(r,c,\mu)}\bigg].
\end{align} 
To bound the inner sum, we use $(r-1)/(r+1) \leq 1$ and obtain
\begin{align*}
    f(r,c,\mu) &\leq \frac{c^r}{(r+1)!} \cdot \frac{r-1}{ (1+(\mu - 1) \cdot r)} \leq \frac{c^r}{r!}  \cdot \frac{1}{1+(\mu-1) \cdot r} \leq  \frac{c^{r}}{r!} \cdot \frac{1}{\mu-1} .
\end{align*} 
%For $r \geq 2$, we plug this bound into~\eqref{eq:DeltaEA2}. For $r=1$, we simply use $f(r,c,\mu)=0$. Moreover, summing to $\infty$ instead of $(1-\eps)n$ only makes the expression in~\eqref{eq:DeltaEA2} smaller, and allows us to use the identity $\sum_{i=2}^\infty c^r/r! = e^{c}-c-1$, yielding  
We plug this bound into~\eqref{eq:DeltaEA2}. Moreover, summing to $\infty$ instead of $(1-\eps)n$ only makes the expression in~\eqref{eq:DeltaEA2} smaller, and allows us to use the identity $\sum_{r=1}^\infty c^r/r! = e^{c}-1 \leq e^c$, yielding  
\begin{align*}
    \DeltaEA(\eps) & \geq \mathcal{O}(\epsilon^2) + (1 + o(1))  \epsilon  c  e^{-c}  \Big(1 - \frac{e^{c}}{\mu - 1}\Big).
\end{align*}
If $n$ is large enough and $\eps$ so small that the $\mathcal{O}(\epsilon^2)$ term can be neglected, then by picking $\mu_0 = 2 + e^{c}$ we get $\DeltaEA(\eps) \gtrsim \epsilon c e^{-c} /(e^c + 1) >0$ and therefore we can set for example $\delta = \tfrac12 c e^{-c}/(e^c + 1)$, which concludes the proof.\qed
\end{proof}

\section{Runtime of the \moea close to the optimum}

In the previous sections, we have shown that the \moea have positive drift close to the optimum if the population size is chosen accordingly. In this section, we explain briefly what this result implies for the runtime of these algorithms.  
\begin{theorem}\label{thm:EA}
Assume that the \moea runs on the \dynbv function with constant parameters $c>0$ and $\mu \geq e^c+2$. Let $\eps_0$ be as in Theorem~\ref{thm:DeltaEA} and let $\eps < \eps_0$. If the \moea is started with a population in which all individuals have at most $\eps n$ zero-bits, then whp it finds the optimum in $\mathcal{O}(n \log n)$ steps.
\end{theorem}
\begin{proof}
We only sketch the proof since the argument is mostly standard, e.g.~\cite{lengler2019general}. First we note that the number of generations between two degenerate populations satisfies a exponential tail bound by Lemma~\ref{lem:tailbound}. As an easy consequence, the total number of flipped bits between two degenerate populations also satisfies an exponential tail bound, and so does the difference $|\Phi^t-\Phi^{t+1}|$. This allows us to conclude from the \emph{negative drift theorem}~\cite{doerr2013adaptive,lengler2020drift} that whp $\Phi^t < \eps_0n$ for an exponential number of steps. However, in the range $\Phi^t \in [0,\eps_0 n]$, by Theorem~\ref{thm:DeltaEA} the drift is positive and multiplicative, $\E[\Phi^t-\Phi^{t+1} \mid \Phi^t] \geq \delta \Phi^t/n$. Therefore, by the \emph{multiplicative drift theorem} \cite{doerr2012multiplicative,lengler2020drift} whp the optimum appears among the first $O(n \log n)$ degenerate populations. Again by Lemma~\ref{lem:tailbound}, whp this corresponds to $O(n\log n)$ generations.\qed
\end{proof}

\section{Second-Order Analysis of the Drift for $\mu =2$}\label{sec:second_order}
In this section we investigate the \toea. We will compute a second order approximation of $\EX[\Phi^t -\Phi^{t+1} \mid \Phi^t = \epsilon n]$, that is we will compute the drift up to $\mathcal O(\epsilon^3)$ error terms. This analysis will allow us to prove the following main result.
\begin{theorem}\label{thm:hardestregion}
There are $C>0$, $c^* > 0$ and $\eps^*>0$ such that the \toea with mutation parameter $c^*$ has positive drift $\Delta(c^*,\eps) \geq C$ for all $\eps \in (0,\tfrac12 \eps^*)$ and has negative drift $\Delta(c^*,\eps) \leq  -C$ for all $\eps \in (\tfrac32 \eps^*,2 \eps^*)$.
\end{theorem}
In a nutshell, Theorem~\ref{thm:hardestregion} shows that the hardest part for optimization is not around the optimum. In other words, it shows that the range of efficient parameters settings is larger close to the optimum. We remark that we ``only'' state the result for one concrete parameter $c^*$, but the same argument could be extended to show that the ``range of efficient parameter settings'' becomes larger. Moreover, with standard arguments similar to Theorem~\ref{thm:EA}, which we omit here, it would be possible to translate positive and negative drift into optimization times. I.e., one could show that whp the algorithm has optimization time $O(n\log n)$ if the algorithm is started in the range $\eps \in (0,\eps^*/4)$, but that the optimization time is exponential if it is started in the range $\eps > 2\eps^*$. For the latter, a small interval $(\tfrac32 \eps^*,2 \eps^*)$ of negative drift suffices, since the exponential tail bounds on the step size $|\Phi^t-\Phi^{t+1}|$ are \emph{not} restricted to that interval. They hold uniformly over the whole search space.

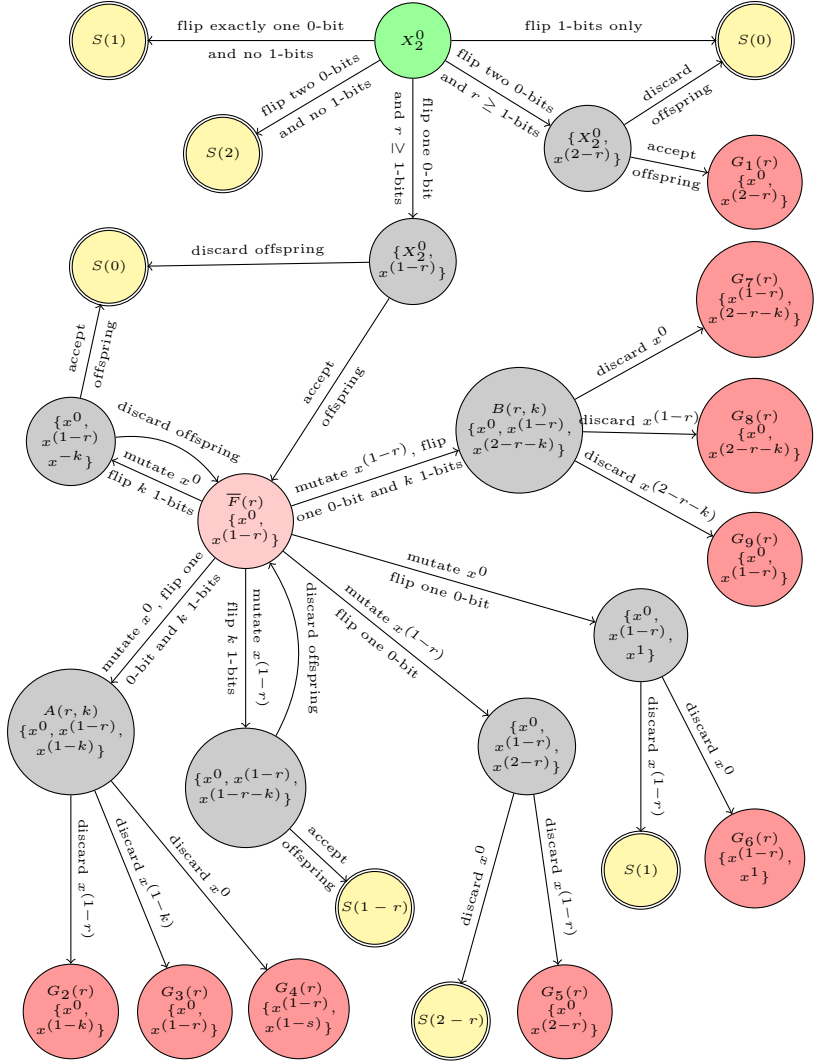
\begin{figure}[h!]
\begin{tikzpicture}[main node/.style={circle,draw,font=\tiny, minimum size = 1cm, inner sep = 0 cm},scale=1]
  
  \node[fill=yellow!40,double] at (-4,12) [main node] (A) {$S(1)$};
  \node[fill=green!40] at (0,12) [main node] (B) {$X_2^0$};
  \node[fill=yellow!40,double] at (4.5,12) [main node] (C) {$S(0)$};
  \node[fill=yellow!40,double] at (-2.5,10.5) [main node] (D) {$S(2)$};
  \node[fill=gray!40] at (2.3,10.5) [main node] (E) {$\begin{array}{c} \{X_{2}^0,\\x^{(2-r)}\} \end{array}$};
  \node[fill=yellow!40, double] at (-4,9) [main node] (F) {$S(0)$};
  \node[fill=red!40] at (4.5,10) [main node] (G) {$\begin{array}{c} G_1(r)\\\{x^0,\\x^{(2-r)}\} \end{array}$};
  \node[fill=gray!40] at (0,9) [main node] (H) {$\begin{array}{c} \{X_2^0,\\x^{(1-r)}\} \end{array}$};
  \node[fill=gray!40] at (-4.5,6.6) [main node] (I) {$\begin{array}{c} \{x^0,\\x^{(1-r)}\\x^{-k}\} \end{array}$};
  \node[fill=red!20] at (-2.2,5.5) [main node] (J) {$\begin{array}{c} \overline F(r)\\\{x^0,\\ x^{(1-r)}\} \end{array}$};
  \node[fill=gray!40] at (1.4,6.5) [main node] (K) {$\begin{array}{c} B(r,k)\\\{x^0,x^{(1-r)},\\x^{(2-r-k)}\} \end{array}$};
  \node[fill=red!40] at (4.5,8.3) [main node] (L) {$\begin{array}{c} G_7(r)\\\{x^{(1-r)},\\x^{(2-r-k)}\} \end{array}$};
  \node[fill=red!40] at (4.5,6.5) [main node] (M) {$\begin{array}{c} G_8(r)\\\{x^0,\\x^{(2-r-k)}\} \end{array}$};
  \node[fill=red!40] at (4.5,5) [main node] (N) {$\begin{array}{c} G_9(r)\\\{x^0,\\ x^{(1-r)}\} \end{array}$};
  \node[fill=gray!40] at (-4.5,2.5) [main node] (O) {$\begin{array}{c} A(r,k)\\\{x^0,x^{(1-r)},\\x^{(1-k)}\} \end{array}$};
  \node[fill=gray!40] at (-2.2,1.8) [main node] (P) {$\begin{array}{c} \{x^0,x^{(1-r)},\\x^{(1-r-k)}\} \end{array}$};
  \node[fill=gray!40] at (3,4) [main node] (Q) {$\begin{array}{c} \{x^0,\\x^{(1-r)},\\x^{1}\} \end{array}$};
  \node[fill=gray!40] at (1.5,2.5) [main node] (R) {$\begin{array}{c} \{x^0,\\x^{(1-r)},\\x^{(2-r)}\} \end{array}$};
  \node[fill=yellow!40,double] at (-0.5,0.5) [main node] (S) {$S(1-r)$};
  \node[fill=red!40] at (4.5,1) [main node] (T) {$\begin{array}{c} G_6(r)\\\{x^{(1-r)},\\ x^{1}\} \end{array}$};
  \node[fill=red!40] at (-4.5,-1) [main node] (U) {$\begin{array}{c} G_2(r)\\\{x^0,\\x^{(1-k)}\} \end{array}$};
  \node[fill=red!40] at (-3,-1) [main node] (V) {$\begin{array}{c} G_3(r)\\\{x^0,\\x^{(1-r)}\} \end{array}$};
  \node[fill=red!40] at (-1.5,-1) [main node] (W) {$\begin{array}{c} G_4(r)\\\{x^{(1-r)},\\ x^{(1-s)}\} \end{array}$};
  \node[fill=yellow!40,double] at (0.5,-1) [main node] (X) {$S(2-r)$};
  \node[fill=red!40] at (2,-1) [main node] (Y) {$\begin{array}{c} G_5(r)\\\{x^0,\\x^{(2-r)}\} \end{array}$};
  \node[fill=yellow!40,double] at (3,1) [main node] (Z) {$S(1)$};
  
  \draw[->] (B) to node [sloped, midway,above ] {\tiny flip exactly one $0$-bit} node [sloped, midway,below ] {\tiny and no $1$-bits} (A);
  \draw[->] (B) to node [sloped, midway,above ] {\tiny flip $1$-bits only} (C);
  \draw[->] (B) to node [sloped, midway,above ] {\tiny flip two $0$-bits} node [sloped, midway,below ] {\tiny and no $1$-bits} (D);
  \draw[->] (B) to node [sloped, midway,above ] {\tiny flip two $0$-bits} node [sloped, midway,below ] {\tiny and $r \ge 1$-bits} (E);
  \draw[->] (E) to node [sloped, midway,above ] {\tiny discard} node [sloped, midway,below ] {\tiny offspring} (C);
  \draw[->] (E) to node [sloped, midway,above ] {\tiny accept} node [sloped, midway,below ] {\tiny offspring} (G);
  \draw[->] (B) to node [sloped, midway,above ] {\tiny flip one $0$-bit} node [sloped, midway,below ] {\tiny and $r \ge 1$-bits} (H);
  \draw[->] (H) to node [sloped, midway,above ] {\tiny discard offspring} (F);
  \draw[->] (H) to node [sloped, midway,above ] {\tiny accept} node [sloped, midway,below ] {\tiny offspring} (J);
  \draw[->] (I) to[bend left] node [sloped, midway,above ] {\tiny discard offspring} (J);
  \draw[->] (J) to node [sloped, midway,above ] {\tiny mutate $x^0$} node [sloped, midway,below ] {\tiny flip $k$ $1$-bits} (I);
  \draw[->] (I) to node [sloped, midway,above ] {\tiny accept} node [sloped, midway,below ] {\tiny offspring} (F);
  \draw[->] (J) to node [sloped, midway,above ] {\tiny mutate $x^{(1-r)}$, flip} node [sloped, midway,below ] {\tiny one $0$-bit and $k$ $1$-bits} (K);
  \draw[->] (K) to node [sloped, midway,above ] {\tiny discard $x^0$} (L);
  \draw[->] (K) to node [sloped, midway,above ] {\tiny discard $x^{(1-r)}$} (M);
  \draw[->] (K) to node [sloped, midway,above ] {\tiny discard $x^{(2-r-k)}$} (N);
  \draw[->] (J) to node [sloped, midway,above ] {\tiny mutate $x^{0}$, flip one} node [sloped, midway,below ] {\tiny $0$-bit and $k$ $1$-bits} (O);
  \draw[->] (O) to node [sloped, midway,above ] {\tiny discard $x^{(1-r)}$} (U);
  \draw[->] (O) to node [sloped, midway,above ] {\tiny discard $x^{(1-k)}$} (V);
  \draw[->] (O) to node [sloped, midway,above ] {\tiny discard $x^{0}$} (W);
  \draw[->] (J) to node [sloped, midway,above ] {\tiny mutate $x^{(1-r)}$} node [sloped, midway,below ] {\tiny flip $k$ $1$-bits} (P);
  \draw[->] (P) to[bend right] node [sloped, midway,above ] {\tiny discard offspring} (J);
  \draw[->] (P) to node [sloped, midway,above ] {\tiny accept} node [sloped, midway,below ] {\tiny offspring} (S);
  \draw[->] (J) to node [sloped, midway,above ] {\tiny mutate $x^{(1-r)}$} node [sloped, midway,below ] {\tiny flip one $0$-bit} (R);
  \draw[->] (R) to node [sloped, midway,above ] {\tiny discard $x^0$} (X);
  \draw[->] (R) to node [sloped, midway,above ] {\tiny discard $x^{(1-r)}$} (Y);
  \draw[->] (J) to node [sloped, midway,above ] {\tiny mutate $x^{0}$} node [sloped, midway,below ] {\tiny flip one $0$-bit} (Q);
  \draw[->] (Q) to node [sloped, midway,above ] {\tiny discard $x^0$} (T);
  \draw[->] (Q) to node [sloped, midway,above ] {\tiny discard $x^{(1-r)}$} (Z);

\end{tikzpicture}

\caption{State diagram for the \toea up to second order.}\label{fig:transitiontoea}
\end{figure}

All this will follow from a second order approximation of the drift, and most of the section is devoted to this end. Let us begin by referring to Figure~\ref{fig:transitiontoea}. 

From the size of the diagram, one can notice how quickly things get complicated further away from the optimum. On a positive note, we can compute the contribution to the drift from many of the states that the population reaches just by using Lemma~\ref{lem:Fr}, which is tight for $\mu=2$. As a reminder, Lemma~\ref{lem:Fr} states that, given a population of two individuals for the \toea, there is a closed formula for the drift, in case there are no more zero-bit flips. The cases where there happen further zero-bit flips can be summarized by $\mathcal{O}(\epsilon)$. In order to get a second order approximation for the drift, we can only apply this lemma once the population has already flipped two zero-bits (each of which give a factor $O(\eps)$), so that the error term is $\mathcal O(\eps^3)$. In particular, in Figure~\ref{fig:transitiontoea} we have colored the states after two zero-bit flips in red. These are denoted by $G_i(r)$ for $i \in \{1,2,..,9\}$. 

We begin by giving the intuition on how to compute some of the more challenging transition probabilities. We will often have to compute, given a population of 3 individuals, the probability for each of them to be discarded, or more precisely that it gives the least fitness value according to the \dynbv function in that iteration. To compute these probabilities it is helpful to notice if any individual dominates another one, since then it will not be discarded. To compare the remaining ones, one only needs to consider all the bits in which the two or three individuals are different and do a case distinction on which of these will be in the first relative position after the permutation. Sometimes it is not enough to look at the first position only, as it could happen that two individuals share the same value in that position and only the third is different. \medskip

%Let's make a more concrete example: consider the case where an individual is the only one to have a zero-bit in a particular position. Then if, after the permutation, that position happens to be the most significant (among the positions in which the individuals differ), that individual will be discarded. If on the other hand, two individuals out of the contending three share a zero-bit in a position and that becomes the most significant one after the permutation, then one has to consider the second position to decide which of the remaining two is discarded. 
%
%Therefore to compute the probability that an individual is discarded, we start by noticing if it dominates any other in the population. If it does, the probability that it gets discarded is zero. If not, we focus on the positions where the individuals in the population have different bits. Then for our individual of interest we compute the probability that one of his zero-bits happens to be in the most significant position. At that point, we need to check if any other individual in the population shared any of those zero-bits. If so, we just need to compare our individual with it, in particular we only have to consider the positions where the two remaining search points differ and compute the probability that one of our individual's zero-bits are sorted in second position. 

%\newpage
%\clearpage
%\includepdf[pages={-}, scale=0.7, pagecommand={\null\vfill\captionof{figure}{State diagram for the \toea further away from the optimum.}}]{(2+1)EA_second_order_Diagram.pdf}
%\newpage
%\clearpage
The first goal will be to compute the drift from state $\overline F(r)$, depicted in light red in Figure~\ref{fig:transitiontoea}. This state is reached if exactly one one-bit and $r\ge 1$ zero-bits are flipped, and the offspring $x^{(1-r)}$ is accepted. From $\overline F(r)$, we can reach two states $A(r,k)$ and $B(r,k)$ by mutating $x^0$ or $x^{(1-r)}$, respectively, and flipping one zero-bit and $k\geq 1$ one-bits. We will start our analysis by computing the contribution to the drift once the population reaches states $A(r,k)$ and $B(r,k)$. For brevity, we denote
\[\Delta_A := \Delta(A(r,k), \eps), \quad \Delta_B := \Delta(B(r,k), \eps), \quad \Delta_i := \Delta(G_i(r), \eps) \text{ for $i=1,\ldots,9$}\]
%\begin{itemize}
%    \item $\Delta_A := \text{contribution to the drift from state $A(r,k)$}$
%    \item $\Delta_B := \text{contribution to the drift from state $B(r,k)$}$
%    \item $\Delta_i := \text{contribution to the drift from state $G_i(r)$}$
%\end{itemize}
To ease reading, we simply write the probability of discarding an individual $x$ as $\Pr(\text{discard $x$})$, without specifying the rest of the population. From Figure~\ref{fig:transitiontoea} it is clear that:
\begin{align*}
    \Delta_A = \Pr(\text{discard $x^{(1-r)}$}) \cdot \Delta_2 + 
    \Pr(\text{discard $x^{(1-k)}$}) \cdot \Delta_3 + 
    \Pr(\text{discard $x^{0}$}) \cdot \Delta_4 
\end{align*}
As discussed at the beginning of this section, we can simply use Lemma~\ref{lem:Fr} to compute:
\begin{align}\label{eq:Delta234}
    \Delta_2 &= \frac{1-k}{k+1}  + \mathcal{O}(\epsilon) \nonumber \\
     \Delta_3 &= \frac{1-r}{r+1}  + \mathcal{O}(\epsilon)\\
     \Delta_4 &= \frac{k+1}{k+r+2} \cdot (1-r) + \frac{r+1}{k+r+2} \cdot (1-k) + \mathcal{O}(\epsilon)= \frac{2-2rk}{k+r+2}  + \mathcal{O}(\epsilon)\nonumber
    \end{align}
Next up, are the probabilities to discard each individual. For that, we will introduce some notation similar as in Definition~\ref{def:B_i}. We sort the positions descendingly according to the next fitness function $f^{t}$. For $i \in \{r,k\}$, the following notation applies to state $A(r,k)$ and with respect to $f^t$.
\begin{itemize}
    \item $F_3 := \text{first among the $r+k+2$ positions in which $x^{0}$, $x^{(1-k)}$ and $x^{(1-r)}$ differ}$.
    \item $F_2^i := \text{first among the $i+1$ positions in which $x^{0}$ and $x^{(1-i)}$ differ}$.
    \item $B_i^0 := \text{set of the i positions where $x^{(1-i)}$ has additional zero-bits over the others}$.
    \item $B_i^1 := \text{position where $x^{(1-i)}$ has the single additional one-bit over the others}$.
\end{itemize}
%In the definitions of $F_3$ and $F_2^i$, the most important bit is given according to the random permutation of the \dynbv function. 
The probability that $x^{(1-r)}$ is discarded can be computed in the same way as in the proof of Lemma~\ref{lem:Fr}: 
\begin{align}\label{eq:discardr}
    \Pr(\text{discard $x^{(1-r)}$}) &= \Pr(F_3 \in B_r^0) + \Pr(F_3 = B_k^1) \cdot \Pr(F_2^r \in B_r^0 \mid F_3 = B_k^1)\nonumber\\
    & = \frac{r}{r+k+2} + \frac{1}{r+k+2} \cdot \frac{r}{r+1} = \frac{r(r+2)}{(r+k+2)(r+1)}.
\end{align}
%Note that $\Pr(F_2^r \in B_r^0 \mid F_3 = B_k^1)$ is the same as the probability that $x^{0}$ is fitter than $x^{(1-r)}$, since we already know that $x^{(1-k)}$ is the fittest. 
Similarly, we have:
\begin{align}\label{eq:discardk}
    \Pr(\text{discard $x^{(1-k)}$}) &= \Pr(F_3 \in B_k^0) + \Pr(F_3 = B_r^1) \cdot \Pr(F_2^k \in B_k^0 \mid F_3 = B_r^1)\nonumber\\
    & = \frac{k}{r+k+2} + \frac{1}{r+k+2} \cdot \frac{k}{k+1} = \frac{k(k+2)}{(r+k+2)(k+1)}.
\end{align}
and
\begin{align}\label{eq:discard0}
    \Pr(\text{discard $x^{0}$}) &= \Pr(F_3 = B_k^1) \cdot \Pr(F_2^r = B_r^1 \mid F_3 = B_k^1) \nonumber\\
     & \quad + \Pr(F_3 = B_r^1) \cdot \Pr(F_2^k = B_k^1 \mid F_3 = B_r^1)\nonumber\\
    & = \frac{1}{r+k+2} \cdot \frac{1}{r+1} + \frac{1}{r+k+2} \cdot \frac{1}{k+1} = \frac{1}{(r+1)(k+1)}.
\end{align}
Putting~\eqref{eq:Delta234}, \eqref{eq:discardr}, \eqref{eq:discardk} and \eqref{eq:discard0} together yields the drift $\Delta_A$:
\begin{align*}
    \Delta_A &= \Pr(\text{discard $x^{(1-r)}$}) \cdot \Delta_2 + \Pr(\text{discard $x^{(1-k)}$}) \cdot \Delta_3 + \Pr(\text{discard $x^{0}$}) \cdot \Delta_4 \\
    & = \mathcal{O}(\epsilon) +\frac{r(r+2)(1-k) + k(k+2)(1-r) + 2-2rk}{(r+k+2)(r+1)(k+1)}.
%    & = \left(\frac{r}{r+k+2} + \frac{1}{r+k+2} \cdot \frac{r}{r+1}\right) \cdot \left[\frac{1}{k+1} \cdot (1-k) + \mathcal{O}(\epsilon)\right] \\
%    & + \left(\frac{k}{r+k+2} + \frac{1}{r+k+2} \cdot \frac{k}{k+1}\right) \cdot \left[\frac{1}{r+1} \cdot (1-r) + \mathcal{O}(\epsilon) \right] \\
%    & + \left(\frac{1}{r+k+2} \cdot \frac{1}{r+1} + \frac{1}{r+k+2} \cdot \frac{1}{k+1}\right) \cdot \left[\frac{k+1}{k+r+2} \cdot (1-r) + \right.\\
%    & \left. + \frac{r+1}{k+r+2} \cdot (1-k) + \mathcal{O}(\epsilon) \right] \\
%    & = \mathcal{O}(\epsilon) + \frac{r+2}{r+1} \cdot \frac{r}{r+k+2} \cdot \frac{1-k}{k+1} + \frac{k+2}{k+1} \cdot \frac{k}{r+k+2} \cdot \frac{1-r}{r+1} \\
%    & + \frac{r+k+2}{(r+1) \cdot (k+1)} \cdot \frac{(k+1) \cdot (1-r) + (r+1) \cdot (1-k)}{k+r+2} \\ 
%    & = \mathcal{O}(\epsilon) + \frac{k^2 \cdot (1-3 \cdot r) + k \cdot (4-r \cdot (3 \cdot r + 8)) + (r+2)^2}{(k+r+2)(r+1) (k+1) \cdot }
\end{align*}
Following the same exact procedures we can compute $\Delta_B$. In particular, we again have:
\begin{align*}
    \Delta_B = \Pr(\text{discard $x^0$}) \cdot \Delta_7 + \Pr(\text{discard $x^{(1-r)}$}) \cdot \Delta_8 + \Pr(\text{discard $x^{(2-r-k)}$}) \cdot \Delta_9
\end{align*}
Note the abuse of notation, where we omitted the rest of the population. In particular, the above probabilities are not the same as in the previous part of the proof, since the underlying population is different. We begin by applying Lemma~\ref{lem:Fr}, which yields:
\begin{align}\label{eq:Delta789}
    \Delta_7 &= \frac{k}{k+1} \cdot (1-r) + \frac{1}{k+1} \cdot (2-r-k) + \mathcal{O}(\epsilon) = \frac{2-r-rk}{k+1} + \mathcal{O}(\epsilon)\nonumber\\
    \Delta_8 &= \frac{2\cdot (2-r-k)}{r+k+2}   + \mathcal{O}(\epsilon) \\
    \Delta_9 &= \frac{1-r}{r+1} + \mathcal{O}(\epsilon)\nonumber
\end{align}
Similar, as before, we sort the positions descendingly according to the current fitness function $f^t$. In the following, the last three definitions are identical as above and are only restated for convenience: 
\begin{itemize}
     \item $\hat{F}_3 := \text{first among the $r+k+2$ positions in which $x^{0}$, $x^{(1-r)}$ and $x^{(2-r-k)}$ differ.}$ 
    \item $\hat{F}_2^{r+k} := \text{first among the $k+1$ positions in which $x^{(1-r)}$ and $x^{(2-r-k)}$ differ.}$ 
    \item $\hat{B}_k^0 := \text{set of the $k$ positions where $x^{(2-r-k)}$ has additional zero-bits over the others.}$
    \item $\hat{B}_k^1 := \text{position where $x^{(2-r-k)}$ has the single additional one-bit over the others.}$
    \item $F_2^r := \text{first among the $r+1$ positions in which $x^{0}$ and $x^{(1-r)}$ differ.}$
    \item $B_r^0 := \text{set of the r positions where $x^{(1-r)}$ has additional zero-bits over $x^0$.}$
    \item $B_r^1 := \text{position where $x^{(1-r)}$ has the single additional one-bit over $x^0$.}$
\end{itemize}
%In the definitions of $\hat{F}_3$ and $\hat{F}_2^{r+k}$, the most important bit is given according to the random permutation of the \dynbv function. 
%We will also use some notation defined before, namely $F_2^r$ (to compare $x^0$ and $x^{(1-r)}$), $B_r^1$ and $B_r^0$. 
We can follow the same reasoning as before and compute:
\begin{align}\label{eq:discard0B}
    \Pr(\text{discard $x^{0}$}) &= \Pr(\hat{F}_3 = B_r^1) + \Pr(\hat{F}_3 = \hat{B}_k^1) \cdot \Pr(F_2^r = B_r^1 \mid \hat{F}_3 = \hat{B}_k^1) \nonumber\\
    & = \frac{1}{r+k+2} + \frac{1}{r+k+2} \cdot \frac{1}{r+1} = \frac{r+2}{(r+1)(r+k+2)}.
\end{align}
\begin{align}\label{eq:discardrB}
    \Pr(\text{discard $x^{(1-r)}$}) &= \Pr(\hat{F}_3 \in B_r^0) \cdot \Pr(\hat{F}_2^{r+k} = \hat{B}_k^1 \mid \hat{F}_3 \in B_r^0) +\nonumber \\
    & \quad + \Pr(\hat{F}_3 = \hat{B}_k^1) \cdot \Pr(F_2^r \in B_r^0 \mid \hat{F}_3 = \hat{B}_k^1)\nonumber\\
    & = \frac{r}{r+k+2} \cdot \frac{1}{k+1} + \frac{1}{r+k+2} \cdot \frac{r}{r+1} = \frac{r}{(r+1)(k+1)}.
\end{align}
\begin{align}\label{eq:discardkB}
    \Pr(\text{discard $x^{(2-r-k)}$}) &= \Pr(\hat{F}_3 \in \hat{B}_k^0) + \Pr(\hat{F}_3 \in B_r^0) \cdot \Pr(\hat{F}_2^{r+k} \in \hat{B}_k^0 \mid \hat{F}_3 \in B_r^0)  \nonumber\\
    & = \frac{k}{r+k+2} + \frac{r}{r+k+2} \cdot \frac{k}{k+1} = \frac{k(r+k+1)}{(k+1)(r+k+2)}.
\end{align}
Combining~\eqref{eq:Delta789}, \eqref{eq:discard0B}, \eqref{eq:discardrB} and \eqref{eq:discardkB}, we obtain: 
\begin{align*}
    \Delta_B &= \Pr(\text{discard $x^0$}) \cdot \Delta_7 + \Pr(\text{discard $x^{(1-r)}$}) \cdot \Delta_8 + \Pr(\text{discard $x^{(2-r-k)}$}) \cdot \Delta_9  \\
    & = \mathcal{O}(\epsilon) + \frac{(r+2)(2-r-rk) + 2r(2-r-k) + k(1-r)(r+k+1)}{(k+1)(r+1)(k+r+2)}\\
%    & = \left(\frac{1}{r+k+2} + \frac{1}{r+k+2} \cdot \frac{1}{r+1} \right) \cdot \left[\frac{k}{k+1} \cdot (1-r) + \frac{1}{k+1} \cdot (2-r-k) + \mathcal{O}(\epsilon) \right] \\
%    & + \left(\frac{r}{r+k+2} \cdot \frac{1}{k+1} + \frac{1}{r+k+2} \cdot \frac{r}{r+1} \right) \cdot \left[\frac{2}{r+k+2} \cdot (2-r-k) + \mathcal{O}(\epsilon) \right] \\
%    & + \left(\frac{k}{r+k+2} + \frac{r}{r+k+2} \cdot \frac{k}{k+1} \right) \cdot \left[\frac{1}{r+1} \cdot (1-r) + \mathcal{O}(\epsilon) \right] \\
%    & = \mathcal{O}(\epsilon) + \frac{r+2}{r+1} \cdot \frac{1}{r+k+2} \cdot \frac{2 - r - r \cdot k}{k+1} + \frac{r}{r+k+2} \cdot \frac{r+k+2}{(r+1) \cdot (k+1)} \cdot \frac{2 \cdot (2-r-k)}{r+k+2} \\
%    & + \frac{k}{r+k+2} \cdot \frac{r+k+1}{k+1} \cdot \frac{1-r}{r+1} \\
    & = \mathcal{O}(\epsilon) + \frac{-2 r^2 k  -r  k^2 -3  r^2- 4  r  k  + k^2 + 4  r+ k + 4}{(k+1)  (r+1)  (k+r+2)}.
\end{align*}
Next up, we compute the contribution to the drift $\Delta_F := \Delta(F(r),\eps)$ from state $F(r)$. Using Lemma~\ref{lem:Fr}, we get: 
\begin{align}\label{eq:Delta56}
    \Delta_5 & = \frac{2\cdot(2-r)}{2+r} + \mathcal{O}(\epsilon).\nonumber\\
    \Delta_6 & = \frac{r+1}{r+2} \cdot (1-r) + \frac{1}{r+2} \cdot 1 + \mathcal{O}(\epsilon) = \frac{2-r^2}{r+2} + \mathcal{O}(\epsilon).
\end{align}
To compute $\Delta_F$, we first name and compute some probabilities for the outcome of a mutation.
\begin{itemize}
    \item $p_0\ \ := \Pr(\text{flip exactly one zero-bit}) = (1+o(1))  c  \epsilon  e^{-c}$ .
%    \item $\Pr(\text{flip one zero-bit from $x^{(1-r)}$}) = p_0 = (1+o(1))  \frac{1}{2}  c  \epsilon  e^{-c}$.
    \item $p^{k \cdot 1} := \Pr(\text{flip $k$ one-bits}) = (1+o(1))  \frac{c^k}{k!}  e^{-c}$.
%    \item $\Pr(\text{flip k one-bits from $x^{0}$}) = p^{k \cdot 1} = (1+o(1))  \frac{1}{2}  \frac{c^k}{k!}  e^{-c}$.
    \item $p_0^{k \cdot 1} := \Pr(\text{flip one zero-bit and $k$ one-bits}) = (1+o(1))  \epsilon  \frac{c^{k+1}}{k!}  e^{-c}$.
%    \item $\Pr(\text{flip one zero-bit and k one-bits from $x^{(1-r)}$}) = p_0^{k \cdot 1} = (1+o(1))  \frac{1}{2}  c  \epsilon  \frac{c^k}{k!}  e^{-c}$.
\end{itemize}
We are finally ready to compute $\Delta_F$ from Figure~\ref{fig:transitiontoea}. As usual, Lemma~\ref{lem:singlezerobit} allows us to summarize the contribution of all states that are not shown in Figure~\ref{fig:transitiontoea} by $\mathcal O(\eps^2)$. The factor $1/2$ comes from the choice of the parent $x^0$ or $x^{(1-r)}$, and the inner factors $1/(r+1)$, $r/(r+1)$ etc. correspond the probabilities of following the subsequent arrows as depicted in Figure~\ref{fig:transitiontoea}. The six summands correspond to the cases of flipping exactly one zero-bit (in $x^0$ or $x^{(1-r)}$), flipping $k$ one-bits (in $x^0$ or $x^{(1-r)}$), and flipping one zero-bit and $k$ one-bits (in $x^0$ or $x^{(1-r)}$), in this order.
\begin{align*}
    \Delta_{F} &= \mathcal{O}(\epsilon^2) + \frac{1}{2}p_0 \left[\tfrac{1}{r+1} \cdot \Delta_6 + \tfrac{r}{r+1} \cdot 1 \right] + \frac{1}{2}p^{k \cdot 1}  \left[\tfrac{r}{r+1} \cdot \Delta_5 + \tfrac{1}{r+1} \cdot (2-r) \right] + \\
    & + \sum_{k=0}^{(1-\epsilon)  n} \bigg(\frac{1}{2}p^{k\cdot 1}  \left[\tfrac{r}{r+k+1} \cdot 0  + \tfrac{k+1}{r+k+1} \cdot \Delta_F \right] + \\
    & \qquad+ \frac{1}{2}p^{k\cdot 1}  \left[\tfrac{1}{r+k+1} \cdot (1-r) + \tfrac{k+r}{r+k+1} \cdot \Delta_F \right]  + \frac{1}{2}p_0^{k \cdot 1} \Delta_A(k) +  \frac{1}{2}p_0^{k \cdot 1} \Delta_B(k)\bigg). \\
\end{align*}
%where we summarized the contributions of cases where the algorithm flips more than one zero-bit from state $\Delta_F$ by $\mathcal{O}(\epsilon^2)$. For the probabilities above, we have: 
%\begin{itemize}
%    \item $p_0 := \Pr(\text{flip one zero-bit from $x^0$}) = (1+o(1)) \cdot \frac{1}{2} \cdot c \cdot \epsilon \cdot e^{-c}$ .
%    \item $\Pr(\text{flip one zero-bit from $x^{(1-r)}$}) = p_0 = (1+o(1)) \cdot \frac{1}{2} \cdot c \cdot \epsilon \cdot e^{-c}$.
%    \item $p^{k \cdot 1} := \Pr(\text{flip k one-bits from $x^{0}$}) = (1+o(1)) \cdot \frac{1}{2} \cdot \frac{c^k}{k!} \cdot e^{-c}$.
%    \item $\Pr(\text{flip k one-bits from $x^{0}$}) = p^{k \cdot 1} = (1+o(1)) \cdot \frac{1}{2} \cdot \frac{c^k}{k!} \cdot e^{-c}$.
%    \item $p_0^{k \cdot 1} := \Pr(\text{flip one zero-bit and k one-bits from $x^{0}$}) = (1+o(1)) \cdot \frac{1}{2} \cdot c \cdot \epsilon \cdot \frac{c^k}{k!} \cdot e^{-c}$.
%    \item $\Pr(\text{flip one zero-bit and k one-bits from $x^{(1-r)}$}) = p_0^{k \cdot 1} = (1+o(1)) \cdot \frac{1}{2} \cdot c \cdot \epsilon \cdot \frac{c^k}{k!} \cdot e^{-c}$.
%\end{itemize}
We simplify and sort the expression: 
\begin{align*}
    \Delta_F  &= \mathcal{O}(\epsilon^2) + \frac12 p_0  \left[\frac{\Delta_6 + 2-r}{r+1}  + \frac{r(1+\Delta_5)}{r+1} \right] + \Delta_F  \left(\sum_{k=0}^{(1-\epsilon) n} \frac12 p^{k \cdot 1}  \frac{2 k+r+1}{r+k+1} \right) + \\
    & + \sum_{k=0}^{(1-\epsilon)  n} \left(\frac12 p^{k \cdot 1} \frac{(1-r)}{r+k+1} + \frac12p_0^{k \cdot 1}  \left(\Delta_A(k) + \Delta_B(k)\right) \right)\\
\end{align*}
Now we solve for $\Delta_F$ by bringing all $\Delta_F$-terms on the left hand side and dividing by its prefactor. We obtain:
\begin{align*}
    \Delta_F  &= \mathcal{O}(\epsilon^2) + \frac{p_0 \frac{2 + \Delta_6 + r \cdot \Delta_5}{r+1} + \sum_{k=0}^{(1-\epsilon) n} p^{k \cdot 1} \frac{(1-r)}{r+k+1} + p_0^{k \cdot 1}  \left(\Delta_A(k) + \Delta_B(k)\right)}{2-\sum_{k=0}^{(1-\epsilon)  n}  p^{k \cdot 1} \frac{2 k+r+1}{r+k+1}}.
\end{align*}
%which clearly depends on $r$, $c$ and $\epsilon$, therefore we will refer to it as $\Delta_F(r,c,\epsilon)$.
For later reference we note that the $p^{k \cdot 1}$ sum up to one. Thus we can rewrite $2 = \sum_k p^{k \cdot 1} \frac{2r+2k+2}{r+k+1}$, which allows us to rewrite the denominator as $\sum_{k=0}^{(1-\epsilon)  n}  p^{k \cdot 1} \frac{r+1}{r+k+1}$. This elegant trick will allow us some cancellations later. In particular, note that the probabilities $p_0$ and $p_0^{k\cdot 1}$ are in $\mathcal O(\eps)$, so when we ignore those terms, then the complicated sums cancel out and we recover the formula $\Delta_F = (1-r)/(r+1) + \mathcal O(\eps)$ from Lemma~\ref{eq:hatp}.

Lastly, we can find the drift  $\Delta(c, \epsilon)$ from the starting population $X^0_2$. To this end, we need two more probabilities:
%
%We can refer to the notation just introduced to compute $\Delta_F(r,c, \epsilon)$ before, namely we will use $p_0$, $p^{r \cdot 1}$ and $p_0^{r \cdot 1}$. Note, however, that in these probabilities there is also a factor $\frac{1}{2}$, which stands for the probability of selecting an individual in the population $\{x^0,x^{(1-r)}\}$, but in the starting population we only have copies of $x^0$, therefore we will multiply these probabilities by $2$. For our computations we also need these probabilities: 
\begin{itemize}
    \item $p_{2 \cdot 0} := \Pr(\text{flip two zero-bits}) =  (1+o(1)) \frac{1}{2}  \epsilon^2  c^2  e^{-c}$
    \item $p_{2 \cdot 0}^{r \cdot 1} := \Pr(\text{flip two zero-bits and r one-bits}) = (1+o(1)) \frac{1}{2} \epsilon^2  e^{-c} \frac{c^{r+2}}{r!}$
\end{itemize}
Again following Figure~\ref{fig:transitiontoea}, we can calculate $\Delta(c,\eps)$. In the following calculation, we use the full notation $\Delta_{F(r)} = \Delta_F$ to make the dependency on $r$ explicit. Moreover, note that we already have computed the drift from state $G_1(r)$, since this is identical with $G_5(r)$, so the drift is $\Delta_5$.
\begin{align*}
    \Delta(c,\epsilon) &= \mathcal{O}(\epsilon^3) + p_0 \cdot 1 + p_{2 \cdot 0} \cdot 2 + \sum_{r=0}^{(1-\epsilon)  n} p^{r \cdot 1} \cdot 0 + \\
    & \quad + \sum_{r=1}^{(1-\epsilon)  n}\left( p_0^{r \cdot 1} \big(\tfrac{1}{r+1} \cdot \Delta_{F(r)} + \tfrac{r}{r+1} \cdot 0\big) + p_{2 \cdot 0}^{r \cdot 1} \big(\tfrac{2}{2+r} \cdot \Delta_5 + \tfrac{r}{r+2} \cdot 0 \big)\right) \\
    & = \mathcal{O}(\epsilon^3) + p_0 + 2 p_{2 \cdot 0} + \sum_{r=1}^{(1-\epsilon)  n} \frac{p_0^{r \cdot 1}}{r+1}  \Delta_{F(r)} + \frac{2p_{2 \cdot 0}^{r \cdot 1}}{2+r} \Delta_5 .
\end{align*}
%where we have $\Delta_1 = \Delta_5$ and we used that the contribution to the drift in case the algorithm immediately flips three zero-bits by $\mathcal{O}(\epsilon^3)$. 

\begin{figure}[h!]
%\includepdf[scale=0.8, pagecommand={\null\vfill\captionof{figure}{Insert new caption here}}]{(mu+1)EA_Diagram.pdf}\label{fig:transitionmoea}
\includegraphics[width = 1.0\textwidth]{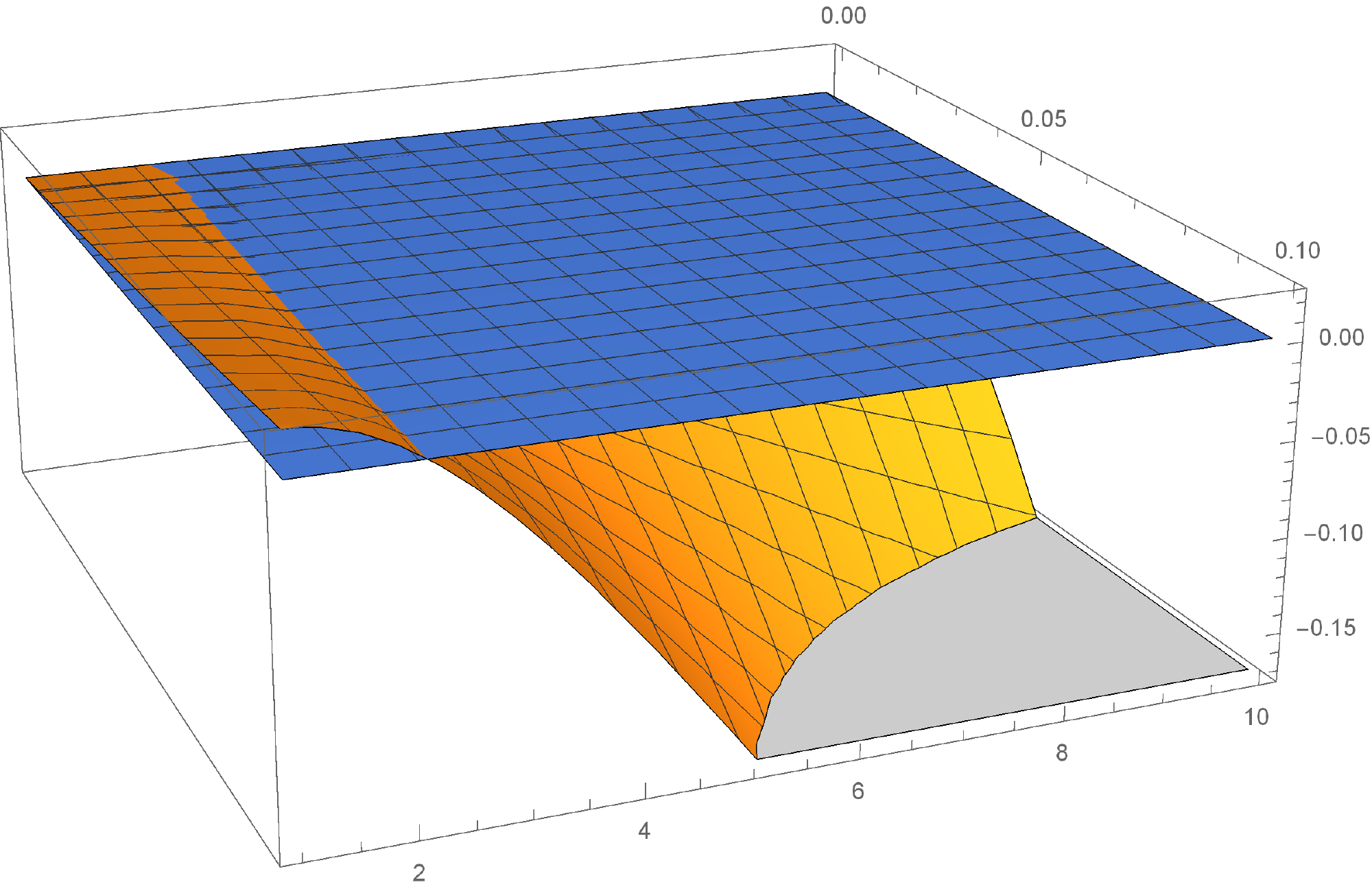}
\caption{Plot of the second order approximation (orange) for the drift of the \toea. The x-axis is the mutation parameter $c$, the y-axis is the distance $\epsilon$ from the optimum. The blue plane is the 0 plane. The interesting part is the line of intersection between the blue and orange surface, as this is boundary between positive and negative drift. Looking closely, the intersection moves to the left (smaller $c$) if we move to the front (larger $\eps$). Thus the problem becomes harder (smaller threshold for $c$) as we increase $\eps$. Hence, the hardest part is not around the optimum. In particular, for some choices of the mutation parameter $c$ (e.g., $c=2.2$) the drift of the \toea is positive in a region around the optimum, but is negative further away from the optimum.  We prove this surprising result below, and it is in line with the experimental results found in \cite{lengler2020large}.}\label{fig:plot}
\end{figure}

The last step is to plug in the formulas for the probabilities and sort terms. Also, letting the sums go to $\infty$ instead of $(1-\eps)n$ will only add another factor of $(1+o(1))$. 
%We will use the elegant trick that, since the probabilities $p^{k \cdot 1}$ sum up to one, we can write $2 = \sum_{k=0}^{(1-\epsilon) n} p^{k \cdot 1} \cdot \frac{2 \cdot k+r+1}{r+k+1} = \sum_{k=0}^{(1-\epsilon) n} p^{k \cdot 1} \cdot \frac{2 \cdot k+r+1}{r+k+1}$ 
Thus we can rewrite the drift as:
\begin{align}\label{eq:secondorderdelta}
    \Delta(c, \epsilon) &= \epsilon  (1+o(1))  f_0(c) + \epsilon^2  (1+o(1))  f_1(c) + \mathcal{O}(\epsilon^3),
\end{align}
where
\begin{align}\label{eq:f0}
    f_0(c) &=  c  e^{-c} + \sum_{r=1}^{\infty} \frac{c^{r+1}}{r!} e^{-c} \cdot \frac{1}{r+1} \cdot \frac{\sum_{k=0}^{\infty} p^{k \cdot 1}  \frac{(1-r)}{r+k+1}}{\sum_{k=0}^{\infty} p^{k \cdot 1}  \frac{r+1}{r+k+1}} \nonumber\\
%    & =  c e^{-c} + \sum_{r=1}^{(1-\epsilon) n} 
%     c \cdot \frac{c^r}{r!} \cdot e^{-c} \cdot \frac{1-r}{(r+1)^2} \\
    & = c  e^{-c} \cdot \bigg(1+\sum_{r=1}^{\infty} \frac{c^r}{(r+1)!} \cdot \frac{1-r}{r+1} \bigg)
\end{align}
%In the second equality we used that $\frac{\sum_{k=0}^{(1-\epsilon) \cdot n} p^{k \cdot 1} \cdot \frac{(1-r)}{r+k+1}}{1-\sum_{k=0}^{(1-\epsilon) \cdot n} p^{k \cdot 1} \cdot \frac{2 \cdot k+r+1}{r+k+1}} = \frac{1-r}{r+1}$, which can be computed using either by applying Lemma~\ref{lem:Fr} or directly as in \cite{lengler2020large}.
and
\begin{align}\label{eq:f1}
    f_1(c) 
    %&=  c^2   e^{-c} + \\
    %& + \sum_{r=1}^{(1-\epsilon)  n}\frac{1}{\epsilon}  \frac{c^{r+1}}{r!} \cdot \frac{e^{-c}}{r+1} \cdot \frac{ce^{-c}  \frac{2 + \Delta_6 + r \cdot \Delta_5}{r+1} + \sum_{k=0}^{(1-\epsilon)  n} p_0^{k \cdot 1}  \left(\Delta_A(k) + \Delta_B(k)\right)}{1-\sum_{k=0}^{(1-\epsilon)  n} p^{k \cdot 1}  \frac{2 \cdot k+r+1}{r+k+1}} \\
    %& + \sum_{r=1}^{(1-\epsilon)  n}  c^2   e^{-c}  \frac{c^r}{r!}  \frac{1}{2+r}  \Delta_1 \\
    & = c^2   e^{-c} + \sum_{r=1}^{\infty} (r+1)   e^{-c}  \frac{c^{r+2}}{(r+2)!}  \Delta_5 \nonumber \\
    & + \frac{e^{-2 \cdot c}}{2}  \sum_{r=1}^{\infty} \frac{c^{r+2}}{(r+1)!}  \cdot \frac{\frac{2 + \Delta_6 + r  \Delta_5}{r+1} + \sum_{k=0}^{\infty} \frac{c^k}{k!}  \left(\Delta_A(k) + \Delta_B(k)\right)}{\sum_{k=0}^{\infty} \frac{c^k}{k!}e^{-c}  \frac{r+1}{r+k+1}}.
\end{align}

After all this preliminary calculations, we are now ready to prove Theorem~\ref{thm:hardestregion}. Moreover, we plot the second order approximation of the drift numerically with Wolfram Mathematica in Figure~\ref{fig:plot}. 
\begin{proof}[of Theorem~\ref{thm:hardestregion}]
Recall that the second order approximation of the drift is given by~\eqref{eq:secondorderdelta}, \eqref{eq:f0} and \eqref{eq:f1}. Inspecting~\eqref{eq:f0}, we see that the sum goes over negative terms, except for the term for $r=1$ which is zero. Thus the factor in the bracket is strictly decreasing in $c$, ranging from 1 (for $c=0$) to $-\infty$ (for $c \to \infty$). In particular, there is exactly one $c_0 >0$ such that $f_0(c_0) = 0$. Numerically we find $c_0= 2.4931\ldots$ and $f_1(c_0) = -0.4845 \ldots < 0$.  

In the following, we will fix some $c^* < c_0$ and set $\eps^* := -f_0(c^*)/f_1(c^*)$. Note that by choosing $c^*$ sufficiently close to $c_0$ we can assume that $f_1(c^*) < 0$, since $f_1$ is a continuous function. Due to the discussion of $f_0$ above, the choice $c^* < c_0$ also implies $f_0(c^*)>0$. Thus $\eps^* >0$. Moreover, since $f_0(c) \to 0$ for $c\to c^*$, if we choose $c^*$ close enough to $c_0$ then we can make $\eps^*$ as close to zero as we wish. 

To add some intuition to these definitions, note that $\Delta(c,\eps) = \eps(f_0(c) + \epsilon f_1(c)+ \mathcal O(\eps^2))$, so the condition $\epsilon = -f_0(c)/f_1(c)$ is a choice for $\eps$ for which the drift is approximately zero, up to the error term. We will indeed prove that for fixed $c^*$, the sign of the drift switches around $\eps \approx \eps^*$. More precisely, we will show that the sign switches from positive to negative as we go from $\Delta(c^*,\epsilon^* - \epsilon{'})$ to $\Delta(c^*,\epsilon^*+ \epsilon{'})$, for $\epsilon' \in (0,\epsilon^*)$. Actually, we will constrict to $\eps' \in (\eps^*/2,\eps^*)$ so that we can handle the error terms. This implies that the value $c^*$ yields positive drift close to the optimum (in the range $\eps \in (0,\tfrac12\eps^*)$), but yields negative drift further away from the optimum (in the range $\eps \in (\tfrac32\eps^*,2\eps^*)$). This implies Theorem~\ref{thm:hardestregion}.

To study the sign of the drift, we define
\begin{align*}
    \Delta^*(c,\epsilon) & := \frac{\Delta(c, \epsilon)}{\epsilon} = (1+o(1)) \cdot \left(f_0(c) + \epsilon \cdot f_1(c) + O(\eps^2)\right).
\end{align*}
It is slightly more convenient to consider $\Delta^*$ instead of $\Delta$, but note that both terms have the same sign. So it remains to investigate the sign of $\Delta^*(c^*,\epsilon^* - \epsilon{'})$ and $\Delta^*(c^*,\epsilon^* + \epsilon{'})$ for $\eps' \in (\eps^*/2,\eps^*)$. We will only study the first term, the second one can be analyzed analogously. Recalling the definition of $\eps^*$ and that $f_1(c^*)<0$, we have
\begin{align*}
    \Delta^*(c^*,\epsilon^*- \epsilon{'}) &= (1+o(1))  \left(f_0(c^*) + (\epsilon^*- \epsilon{'})  f_1(c^*) \right) + \mathcal O((\eps^*)^2)\\
    & = (1+o(1))  \big(\underbrace{f_0(c^*) + \epsilon^*  f_1(c^*)}_{=0}\big) - (1+o(1))  \underbrace{\big(\epsilon{'}  f_1(c^*)\big)}_{<\eps^*f_1(c^*)/2}  + \mathcal O((\eps^*)^2)\\
    & > -(1+o(1))\tfrac12 \eps^*f_1(c^*) + \mathcal O((\eps^*)^2).
\end{align*}
Recall that we may choose $\eps^*$ as small as we want. In particular, we can choose it so small that the above term has the same sign as the main term, which is positive due to $f_1(c^*)<0$. Hence $\Delta^*(c^*,\epsilon^* - \epsilon{'}) >0$, as desired. Also note that the lower bound is independent of $\eps'$, i.e., it holds uniformly for all $\eps'\in(\eps^*/2,\eps^*)$, which corresponds to the argument $\eps^*-\eps'$ of $\Delta^*$ to be in the interval $(0,\eps^*/2)$. The inequality $\Delta^*(c^*,\epsilon^* + \epsilon{'}) <0$ follows analogously. This concludes the proof.\qed
\end{proof}

%%%%%%%%%%%%%%%%%%%%% CONCLUSION %%%%%%%%%%%%%%%%%%

\section{Conclusion}
We have explored the \dynbv function, and we have found that the \moea profits from large population size, close to the optimum. In particular, for all choices of the mutation parameter $c$, the \moea is efficient around the optimum if $\mu$ is large enough. However, surprisingly the region around the optimum may not be the most difficult region. For $\mu=2$, we have proven that it is not.

This surprising result, in line with the experiments in \cite{lengler2020large}, raises much more questions than it answers. Does the \moea with increasing $\mu$ turn efficient for a larger and larger ranges of $c$, as the behavior around the optimum suggests? Or is the opposite true, that the range of efficient $c$ shrinks to zero as the population grows, as it is the case for the \moea on \hottopic functions? Where is the hardest region for larger $\mu$? Around the optimum or elsewhere?

For the \moga, the picture is even less complete. Experiments in \cite{lengler2020large} indicated that the hardest region of \dynbv for the \moga is around the optimum, and that the range of efficient $c$ increases with $\mu$. But the experiments were only run for $\mu \leq 5$, and formal proofs are missing. Should we expect that the discrepancy between \moga (hardest region around optimum) and \moea (hardest region elsewhere) remains if we increase the population size, and possibly becomes stronger? Or does it disappear? For \hottopic functions, we know that around the optimum, the range of efficient $c$ becomes arbitrarily large as $\mu$ grows (similarly as we have shown for the \moea on \dynbv), but we have no idea what the situation away from the optimum is. 

The similarities of results between \dynbv and \hottopic functions are striking, and we are pretty clueless where they come from. For example, the analysis of the \moea on \hottopic away from the optimum in \cite{lengler2019exponential} clearly does not generalize to \dynbv since the very heart of the proof is that the weights do not change over long periods. In \dynbv, they change every round. Nevertheless, experiments and theoretical results indicate that the outcome is similar in both cases. Perhaps one could gain insight from ``interpolating'' between \dynbv and \hottopic by re-drawing the weights not every round, but only every $k$-th round. 

 In general, the situation away from the optimum is governed by complex population dynamics, which is why the \moea and the \moga might behave very differently. Currently, we lack the theoretic means to understand population dynamics in which the internal population structure is complex and essential. The authors believe that developing tools for understanding such dynamics is one of the most important projects for improving our understanding of population-based search heuristics.

%For future work, there are many unanswered questions. First, we are unsure of the influence of larger populations further away from the optimum, in particular, whether the hardest region for optimization is close to the optimum for the GA and further away for the EA. We also lack any proof that crossover helps for larger populations, although previous results seem to indicate so. Furthermore, there are many possible different variations of the \dynbv function that would be of interest. In particular, an interesting variant of the DynBV could be when a new permutation is not redrawn each round but only after $k$ rounds.

%%%%%%%%%%%%%%%%%%%%% APPENDIX %%%%%%%%%%%%%%%%%%%%

\bibliographystyle{spmpsci}
\bibliography{refs}
\end{document}